\documentclass[11pt, draftclsnofoot, onecolumn]{IEEEtran}
\usepackage[utf8]{inputenc} % allow utf-8 input
\usepackage[T1]{fontenc}    % use 8-bit T1 fonts
\usepackage{hyperref}       % hyperlinks
\usepackage{url}            % simple URL typesetting
\usepackage{booktabs}       % professional-quality tables
\usepackage{amsfonts}       % blackboard math symbols
\usepackage{nicefrac}       % compact symbols for 1/2, etc.
\usepackage{microtype}      % microtypography
\usepackage{amsmath,amssymb,graphicx,paralist,subfigure,pdfpages,cite,algorithm,algpseudocode,paralist}
\usepackage{setspace}
\usepackage{amsthm}
\usepackage{verbatim}
\usepackage{algorithm}
\usepackage{algpseudocode}
\usepackage{color, colortbl}
\newtheorem{lem}{Lemma} 
\newtheorem{theorem}{Theorem}

\newtheorem{assump}{Assumption}

\def\ln{{\rm ln}}

\def\mc{\mathcal}
\def\mb{\mathbf}
\def\mbb{\mathbb}
\def\ra{\rightarrow}
\def\la{\leftarrow}
\def\P{\mathbf{P}}

\def\bpi{\boldsymbol\pi}

\def\GS{\textbf{GT-SAGA}}
\def\SGD{\textbf{\texttt{SGD}}}
\def\SGP{\textbf{\texttt{SGP}}}
\def\DSGD{\textbf{\texttt{DSGD}}}
\def\DSGT{\textbf{\texttt{DSGT}}}
\def\SG{\textbf{\texttt{SAGA}}}
\def\PS{\textbf{\texttt{Push-SAGA}}}
\def\PDG{\textbf{\texttt{Push-DIGing}}}
\def\GS{\textbf{\texttt{GT-SAGA}}}

\def\GP{\textbf{\texttt{GP}}}
\def\SGP{\textbf{\texttt{SGP}}}
\def\ADD{\textbf{\texttt{ADDOPT}}}
\def\SA{\textbf{\texttt{SADDOPT}}}

\def\mbb{\mathbb}%R
\def\mb{\mathbf}%vector
\def\mc{\mathcal}%set

\def\ol{\overline}
\def\ul{\underline}

\def\bds{\boldsymbol}
\newcommand{\mn}[1]{{\left\vert\kern-0.25ex\left\vert\kern-0.25ex\left\vert\kern0.3ex #1 
		\kern0.3ex\right\vert\kern-0.25ex\right\vert\kern-0.25ex\right\vert}}

\title{\huge \textbf{Push-SAGA: A decentralized stochastic algorithm with variance reduction over directed graphs}}
\author{
Muhammad I. Qureshi$^\dagger$, Ran Xin$^\ddagger$, Soummya Kar$^\ddagger$, and Usman A. Khan$^\dagger$\\
$^\dagger$Tufts University, Medford, MA, USA, $^\ddagger$Carnegie Mellon University, Pittsburgh, PA, USA
\thanks{The authors acknowledge the support of NSF  under awards  CCF-1513936, CMMI-1903972, and CBET-1935555.}
}

\begin{document}

\maketitle

\begin{abstract}
In this paper, we propose Push-SAGA, a decentralized stochastic first-order method for finite-sum minimization over a directed network of nodes. Push-SAGA combines node-level variance reduction to remove the uncertainty caused by stochastic gradients, network-level gradient tracking to address the distributed nature of the data, and push-sum consensus to tackle the challenge of directed communication links. We show that Push-SAGA achieves linear convergence to the exact solution for smooth and strongly convex problems and is thus the first linearly-convergent stochastic algorithm over arbitrary strongly connected directed graphs. We also characterize the regimes in which Push-SAGA achieves a linear speed-up compared to its centralized counterpart and achieves a network-independent convergence rate. We illustrate the behavior and convergence properties of Push-SAGA with the help of numerical experiments on strongly convex and non-convex problems.
\end{abstract}

\section{Introduction}
We consider decentralized finite-sum minimization over a network of~$n$ nodes, i.e.,
\[
\P:\qquad \min_{\mb z\in\mbb R^p} F(\mb{z}) := \frac{1}{n} \sum_{i=1}^{n} f_i(\mb{z}), \qquad f_i(\mb{z}) := \frac{1}{m_i} \sum_{j=1}^{m_i} f_{i,j}(\mb{z}),
\]
where each local cost function~${f_i:\mbb R^p\ra \mbb R}$, private to node~$i$, is further decomposable into~$m_i$ component cost functions. Problems of this nature commonly arise in many training and inference tasks over decentralized and distributed data. In many modern applications, problems of interest have become very large-scale and huge amounts of data is being stored or collected at a large number of geographically distributed nodes (machines, devices, robots). It is thus imperative to design methods that are efficient in both computation and communication.

\begin{figure}[!ht]
\centering
\includegraphics[height=1.2in]{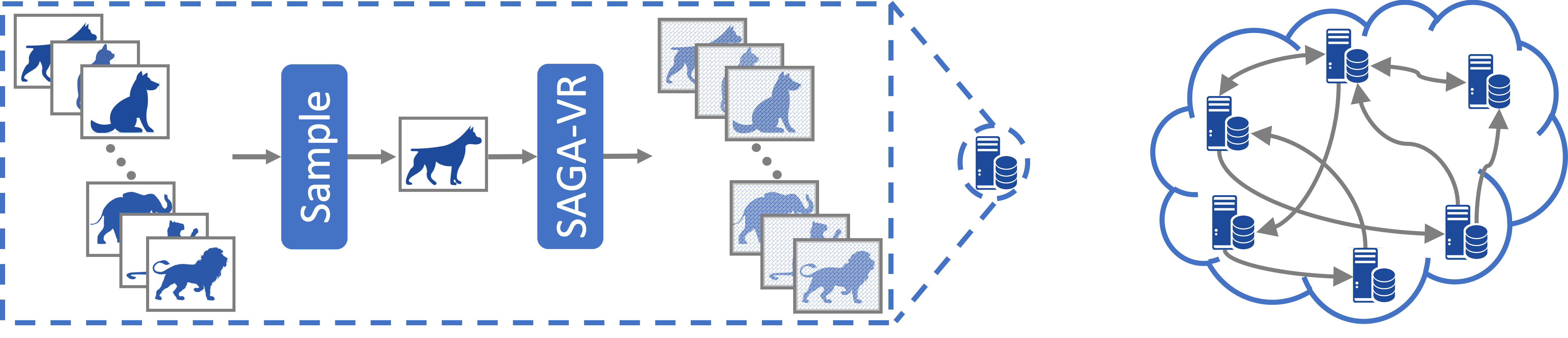}
\caption{(Left) Node-level: Each node computes the gradient at a random data sample and then estimates the local batch gradient with the help of variance reduction. (Right) Network-level: The nodes implement global gradient tracking with the help of inter-node fusion and push-sum.}
\label{network}
\end{figure}

This paper describes a \textit{stochastic, first-order} method~\PS~with a low per-iteration computation complexity, while the nodes communicate over \textit{directed graphs} that are particularly amenable to efficient, resource-constrained network design and often result from severing costly communication links. Existing decentralized stochastic gradient methods over general directed graphs suffer from the variance of the stochastic gradients and the disparity between the local~$f_i$ and global costs~${F={\sum_i} f_i}$. To overcome these challenges,~\PS~utilizes \textit{variance reduction}, locally at each node, to remove the uncertainty caused by the stochastic gradients, and \textit{gradient tracking}, at the network level, to track the global gradient; see Fig.~\ref{network}. Since the underlying communication graph is directed,~\PS~further uses the push-sum protocol to enable agreement among the nodes with network weight matrices that are not necessarily doubly stochastic. 

\subsection{Related Work} 
Decentralized stochastic gradient descent~(\DSGD) over undirected graphs can be found in~\cite{DSGD_Nedich,Diffusion_Chen,DGD_Kar,DSGD_NIPS}. Certain convergence aspects of~\DSGD~are further improved in~\DSGT~\cite{DSGT_Pu}~with the help of gradient tracking~\cite{GT_CDC,DAC,NEXT,harnessing,DIGing}. Of relevance are also~\texttt{Exact Diffusion}~\cite{SED} and~\texttt{D}$^2$~\cite{D2} that are stochastic methods based on an~\texttt{EXTRA}-type bias-correction principle.  For smooth and strongly convex problems,~\DSGD~and~\DSGT, similar to their centralized counterpart~\SGD~\cite{OPTML}, converge linearly to an inexact solution (or sublinearly to the exact solution) due to the variance of the stochastic gradients. Linear convergence in decentralized stochastic algorithms has been shown with the help of various local variance reduction schemes~\cite{SVRG,SAGA,AVRG,APCG}; see related work in~\cite{DSA,DAVRG,DSBA,ADFS}. However, all of these decentralized algorithms require symmetric, doubly stochastic network weights and thus are not applicable to directed graphs. A recent work~\GS~\cite{GTVR} combines variance reduction and gradient tracking to establish linear convergence over weight-balanced directed graphs~\cite{weightbalanced_digraph} with the help of doubly stochastic weight matrices.
% \begin{figure}[!ht]
% \centering
% \includegraphics[height=1in]{Fig3.png}
% \caption{(Left) Node-level: Each node computes the gradient at a random data sample and then estimates the local batch gradient with the help of variance reduction. (Right) Network-level: The nodes implement global gradient tracking with the help of inter-node fusion and push-sum.}
% \label{network}
% \end{figure}

However, not much progress has been made on stochastic methods over arbitrary directed graphs, where doubly stochastic weights cannot be constructed. Related work that does not use doubly stochastic weights includes stochastic gradient push~(\SGP)~\cite{SGP_nedich,SGP_ICML,SGP_AsyncOlshevsky}, that extends~\DSGD~to directed graphs with the help of push-sum consensus~\cite{ac_directed0}, and~\SA~\cite{saddopt} that adds gradient tracking to~\SGP. For smooth and strongly convex problems, both~$\SGP$ and~$\SA$, similar to their undirected counterparts~\DSGD~and~\DSGT, converge linearly to an \textit{inexact} solution with a constant stepsize and sublinearly to the exact solution with decaying stepsizes. Of relevance are also~\GP~\cite{opdirect_Tsianous} and~\PDG/\ADD~\cite{add-opt,DIGing}, which are non-stochastic counterparts of~$\SGP$ and~$\SA$ as they use local full (batch) gradients at each node. See also relevant work in~\cite{AsyncGP_mike,AsyncGP_scutari} on asynchronous implementations of the related (non-stochastic) methods. 

\subsection{Main Contributions.} 
The convergence of~\PS~is formally described in the following.

\begin{theorem} \label{th_main}
Consider Problem~$\P$ and let~${M:=\max_i m_i}$,~${m:=\min_i m_i}$, and each~$f_{i,j}$ be $L$-smooth and each~$f_i$ to be~$\mu$-strongly convex. For the stepsize~$\alpha \in (0, \ol{\alpha})$, for some~$\ol{\alpha}>0$,~\PS~linearly converges, at each node, to the global minimum~$\mb{z}^*$ of~$F$. In particular, for~$\alpha=\ol{\alpha}$,~$\PS$ achieves an~$\epsilon$-optimal solution in
\begin{align*}
\mc{O} \left(\max \left \{ M, \frac{M}{m}\frac{\kappa^2 \psi}{(1-\lambda)^2} \right \} \log \frac{1}{\epsilon} \right),
\end{align*}
component gradient computations (in parallel) at each node,
where~${\kappa:=L/\mu}$ is the condition number of~$F$,~${(1-\lambda)}$ is the spectral gap of the network weight matrix, and~$\psi \geq 1$~is a directivity constant.
\end{theorem}

The main contributions of this paper are summarized next:

\textbf{(1) Linear convergence.} \PS~is the first linearly-convergent \textit{stochastic} method to minimize a finite sum of smooth and strongly convex cost functions over arbitrary \textit{directed} graphs. We emphasize that the analysis of~\PS~does not extend directly from the methods over undirected graphs. This is because: (i) the underlying weight matrices do not contract in the standard Euclidean norm; and, (ii) the algorithm has a nonlinear iterative component due to the push-sum update.

\textbf{(2) Directivity constant.} We explicitly quantify the directed nature of the underlying graphs with the help of a directivity constant~$\psi\geq1$, which equals to~$1$ for undirected and weight-balanced directed graphs, and thus, for finite-sum minimization, this work includes~\DSGD,~\SGP,~\DSGT,~\SA, and~\GS~as its special cases.

\textbf{(3) Linear speed-up and Network-independent convergence.} In a big-data regime where~$${M\approx m \gg\kappa^2\psi(1-\lambda)^{-2}},$$ \PS~with a complexity of~$\mc O(M\log \frac{1}{\epsilon})$ is~$n$ times faster than the centralized~\SG, and this convergence rate is further independent of the network parameters. 

\textbf{(4) Improved Performance.} In the aforementioned big-data regime,~{\PS} improves upon the related linearly-convergent methods~\cite{DSA,DAVRG,AGT} over undirected graphs in terms of the joint dependence on~$\kappa$ and~$m$; with the exception of \textbf{\texttt{DSBA}}~\cite{DSBA} and \textbf{\texttt{ADFS}}~\cite{ADFS}, both of which achieve a better iteration complexity at the expense of computing the proximal mapping at each iteration. 

The rest of the paper is organized as follows. In Section~\ref{algo_dev}, we provide algorithm development and formally describe~\PS. Section~\ref{conv_ana} provides the convergence analysis, while Section~\ref{num_expt} contains numerical experiments on strongly convex and non-convex problems.
\section{Motivation and Algorithm Development}\label{algo_dev}
In order to motivate~\PS, we first describe~\DSGD, a well-known decentralized extension of~\SGD, and its performance with a constant stepsize~$\alpha$. Let~$\mb z^*$ denote the global minimum of Problem~$\P$ and let~$\mb x_i^k\in\mbb R^p$ denote the~\DSGD~estimate of~$\mb z^*$ at node~$i$ and iteration~$k$. Each node~$i$ updates~$\mb x^k_i$ as
\begin{align}\label{dsgd_eq}
\mb x_i^{k+1} = \sum_{r=1}^n w_{ir} \mb x_r^k - \alpha \cdot \nabla f_{i,s_i^k}(\mb x_i^k),\qquad k\geq0,
\end{align}
where~${\ul W=\{w_{ir}\}\in\mbb R^{n\times n}}$ is a network weight matrix that respects the communication graph, i.e.,~${w_{ij}\neq 0}$, if and only if node~$j$ can send information to node~$i$, and~$s_i^k$ is chosen uniformly at random from the set~$\{1,\ldots,m_i\}$ at each iteration~$k$. Under the corresponding smoothness and strong convexity conditions, and assuming that each local stochastic gradient has a bounded variance, i.e.,
\[
\mathbb{E}_{s^k_i}[\|\nabla f_{i,s^k_i}(\mb x^k_i)-\nabla f_i(\mb x_i^k)\|_2^2 ~|\:\mb x_i^k] \leq\sigma^2,\qquad\forall i,k,
\]
it can be shown that, for a certain constant stepsize~$\alpha$, the error~$\mathbb{E}[\|\mb x^k_i-\mb z^*\|_2^2]$, at each node~$i$, decays at a linear rate of~$\left(1-\mc{O}(\mu\alpha)\right)^k$ to a neighborhood of~$\mb z^*$ such that~\cite{SED} 
\begin{equation}\label{DSGD_convergence}
\limsup_{k\rightarrow\infty}\frac{1}{n}\sum_{i=1}^{n}\mathbb{E}[\|\mb x^k_i-\mb z^*\|_2^2]
=  \mc{O}\Big(~\frac{\alpha}{n\mu}\:\sigma^2
+ \frac{\alpha^2\kappa^2}{1-\lambda}\:\sigma^2
+ \frac{\alpha^2\kappa^2} {(1-\lambda)^2}\:\eta~\Big),
\end{equation}
where~$\eta := \frac{1}{n}\sum_{i=1}^{n}\left\|\nabla f_i\left(\mb z^*\right)\right\|_2^2$ and~$(1-\lambda)$ is the spectral gap of the weight matrix~$\ul W$. Equation~\eqref{DSGD_convergence} is based on a constant stepsize~$\alpha$ that leads to a linear but inexact convergence and our goal is to recover linear convergence to the exact solution. (Note that a constant stepsize is essential for linear convergence and a decaying stepsize even though removes the steady state error but the resulting convergence rate is sublinear.)

We now consider the error terms in~\eqref{DSGD_convergence}. The first two terms both depend on the variance~$\sigma^2$ introduced due to the stochastic gradient and vanish as~$\sigma^2\ra0$; a variance reduction scheme that replaces the local stochastic gradients~$\nabla f_{i,s_i^k}$, in~\DSGD~\eqref{dsgd_eq}, with an estimate of the  local batch gradient~${\textstyle\sum_j \!\nabla f_{i,j}}$ thus potentially removes this variance. The last term in~\eqref{DSGD_convergence} involves~$\eta$, which quantifies the disparity between the local costs~$f_i$'s and the global cost~$F$ (recall that~$\nabla F(\mb z^*)=0$). A mechanism that uses the local gradient estimators (from the variance reduction step) to learn the global gradient thus removes~$\eta$; this is realized with the help of dynamic average consensus~\cite{DAC}. In summary, adding local variance reduction and global gradient tracking to~\DSGD~potentially leads to linear convergence for smooth and strongly convex problems. However, the weights~$w_{ir}$ in~\DSGD~are such that~$\ul W=\{w_{ir}\}$ is doubly stochastic, which in general requires the underlying communication graph to be undirected. 

In directed graphs, the weight matrix may either be row stochastic or column stochastic, in general, but not both at the same time. Consequently, the proposed method~\PS~uses primitive, column stochastic weights~${\ul B\in\mbb R^{n\times n}}$, for which it can be verified that the nodes do not reach agreement, i.e.,~$\ul B\mb 1_n\neq\mb 1_n$, where~$\mb 1_n$ is a column vector of~$n$ ones. In fact, assuming~$\bpi$ to be right eigenvector of~$\ul B$ corresponding to the eigenvalue~$1$, the iterations~${\mb x^{k+1}=\ul B\mb x^k \ra \ul B^\infty \mb x^0 = \bpi\mb1_n^\top \mb x^0}$, which only leads to an agreement among the components of~$\mb x^k$, when its~$i$-th component~$\mb{x}^k_i$ is scaled by the~$i$-th component~$\bpi_i$ of~$\bpi$. This asymmetry, caused by the non-$\mb 1_n$ (right) eigenvector of~$\ul B$, is removed with the help of the push-sum correction. In particular, push-sum estimates the non-$\mb 1_n$ eigenvector~$\bpi$ with separate iterations:~${\mb y^{k+1} = \ul B \mb y^k},~{\mb y^0=\mb 1_n}$; subsequently, each component of~$\mb x^k$ is scaled by the corresponding component of~$\mb y^k$ to obtain an agreement in the~$\mb x^k_i / \mb y^k_i$ iterate. 

\subsection{Algorithm Description:~\PS}
The proposed algorithms~\PS~has three main components, see also Fig.~\ref{network}:
\begin{enumerate}[(i)]
\item Variance reduction, which utilizes the~\SG-based gradient estimator~\cite{SAGA} to estimate the local batch gradient~$\nabla f_i$ at each node~$i$ from locally sampled gradients;
\item Gradient tracking, which is based on dynamic average consensus~\cite{DAC} to estimate the global gradient~$\nabla F$ from the local batch gradient estimates; and,
\item Push-sum consensus~\cite{ac_directed0}, which cancels the imbalance caused by the asymmetric nature of the underlying (directional) communication.
\end{enumerate}

The algorithm is formally described next.
\begin{algorithm}[H] 
	\caption{~\textbf{\texttt{Push-SAGA}} at each node~$i$} \label{algo}
	\begin{spacing}{1.4}
    \begin{algorithmic}[1] 
		\Require ${\mb{x}_i^0\in\mbb R^p},~{\mb{z}_i^0 = \mb{x}_i^0},~{\mb{w}_i^0=\mb{g}_i^0=\nabla f_{i}(\mb{z}_{i}^{0})},~{\mb{v}_{i,j}^1 = \mb{z}_i^0,~\forall j \in \{1,\cdots,m_i\}},~{y_i^0=1},~{\{b_{ir}\}_{r=1}^n}$,
		
		~~Gradient table:~$\{\nabla f_{i,j}(\mb{v}_{i,j}^0)\}_{j=1}^{m_i}$~and~${\alpha>0}$
		
		\For{$k = 0,1,2,\cdots$}
		\State$\mb{x}_{i}^{k+1} \la \sum_{r=1}^n b_{ir} \mb{x}_{r}^{k} - \alpha\cdot\mb{w}_{i}^{k}$ %\Comment{Estimate update}
		\State $y_{i}^{k+1} \:\la \sum_{r =1}^n b_{ir}y_{r}^{k}$
		\State $\mb{z}_i^{k+1} \:\la {\mb{x}_i^{k+1}}/{y_i^{k+1}}$ %\Comment{Eigenvalue Estimation and Scaling}	
		\State \textbf{Select}~$s_{i}^{k+1}$ uniformly at random from~$\{1,\cdots,m_i\}$
		%\Comment{Sample from  local data}
		\State $\mb{g}_{i}^{k+1} \la \nabla f_{i,s_i^{k+1}}(\mb{z}_{i}^{k+1}) - \nabla f_{i,s_i^{k+1}}(\mb{v}_{i,s_i^{k+1}}^{k+1}) + \frac{1}{m_i}\sum_{j=1}^{m_i}\nabla f_{i,j}(\mb{v}_{i,j}^{k+1})$ 
		%\Comment{SAGA update}
		\State \textbf{Replace}~$\nabla f_{i,s_i^{k+1}}(\mb{v}_{i,s_i^{k+1}}^{k+1})$ by~$\nabla f_{i,s_i^{k+1}}(\mb{z}_{i}^{k+1})$ in the gradient table
		\State $\mb{w}_{i}^{k+1} \la \sum_{r =1 }^{n}b_{ir} \mb{w}_{r}^{k} + \mb{g}_i^{k+1} - \mb{g}_i^{k}$
		%\Comment{Gradient Tracker update}
		\If{$j = s_{i}^{k+1}$,}
		$\mb{v}_{i,j}^{k+2} \la \mb{z}_i^{k+1}$,~\textbf{else}~$\mb{v}_{i,j}^{k+2} \la \mb{v}_{i,j}^{k+1}$
		\EndIf
		\EndFor
	\end{algorithmic}
	\end{spacing}
\end{algorithm}
\PS~requires a gradient table at each node~$i$, where~$m_i$ component gradients~$\{\nabla f_{i,j}\}_{j=1}^{m_i}$ are stored. At each iteration~$k$, each node~$i$ first computes an~\SGP-type iterate~$\mb z_k^i$ with the help of the push-sum correction. It is important to note that the descent direction in the~$\mb x_k^i$-update (and thus in the~$\mb z_k^i$-update) is~$\mb w_k^i$, which is the global gradient tracker, in contrast to the locally sampled gradient as in~\DSGD~\eqref{dsgd_eq}. Subsequently, node~$i$ generates a random index~$s_i^k$ and computes the \SG-based gradient estimator~$\mb g_i^k$, with the help of the current iterate~$\mb z_i^k$ and the elements from the gradient table. The gradient table is updated next only at the~$s_i^k$-th element, while the other elements remain unchanged. Finally, these gradient estimators~$\mb g^k_i$'s are fused over a network-level update, with the help of dynamic average consensus, to obtain~$\mb w_i^k$'s that track the global gradient.

We remark that the computation and communication advantages of~\PS~are realized at an additional storage requirement. In particular, each node requires~$\mc O(pm_i)$ storage that can be reduced to~$\mc O(m_i)$ for certain problems~\cite{SAGA}. The convergence analysis of~\PS~is provided next. 

\section{Convergence of Push-SAGA}
\label{conv_ana}
This section provides the formal analysis of~\PS. We start with the following assumptions.
\begin{assump}[Column stochastic weights]\label{col_stoc_w}
The weight matrix~$\ul{B}=\{b_{ir}\}\in\mathbb{R}^{n \times n}$ associated with the directed graph is primitive and column stochastic, i.e.,~$\mb{1}_n^\top \ul B = \mb{1}_n^\top$ and~$\ul B\bpi=\bpi$,
where~$\mb 1_n$ is a vector of~$n$ ones and~${\bpi}$ is the right (positive) eigenvector of~$\ul B$ for the eigenvalue~$1$ such that~${\mb 1_n^\top\bpi=1}$. 
\end{assump}
Column-stochastic weights can be locally designed at each node  by choosing~${b_{ri}=1/d_i^{\mbox {\tiny out}}}$, where~$d_i^{\mbox {\tiny out}}$ is the out-degree at node~$i$. From Perron Frobenius theorem, we have~${\ul B^\infty:=\lim_{k\ra\infty}\ul B^k=\bpi \mb 1_n^\top}$. Let~${\|\cdot\|_2}$ and~$\mn{\cdot}_2$ denote the standard vector $2$-norm and the matrix norm induced by it, respectively, and define a weighted inner product as~${\langle\mb{x}, \mb{z}\rangle_{\bds{\pi}} \!:=\! \mb{x}^\top \mbox{diag}(\bds{\pi})^{-1} \mb{z}}$, for~${\mb{x},\mb{z}\in\mbb{R}^p}$, which leads to a weighted Euclidean norm:~${\|\mb{x}\|_{\bds{\pi}} := \|\mbox{diag}(\sqrt{\bds{\pi}})^{-1} \mb{x}\|_2}$. We denote~$\mn{\cdot}_{\bds{\pi}}$ as the matrix norm induced by~${\|\cdot\|_{\bds{\pi}}}$ such that~${\forall X \in \mbb{R}^{n \times n}}$,~${\mn{X}_{\bpi} := \mn{\mbox{diag}(\sqrt{\bds{\pi}})^{-1}X\mbox{diag}(\sqrt{\bds{\pi}})}_2}$. Under this induced norm,~$\ul B$ contracts in the eigenspace orthogonal to the eigenspace corresponding to the eigenvalue~$1$, see~\cite{DSGT_Xin} for formal arguments, i.e., 
\begin{align}
\lambda:=\mn{\ul B-\ul B^\infty}_{\bpi}<1\Rightarrow (1-\lambda)<1,
\end{align}
where~${(1-\lambda)}$ is the spectral gap of the weight matrix~$\ul B$. Moreover, it can be verified from the norm definitions that~${\|\cdot\|_{\bds{\pi}} \leq \ul{{\pi}}^{-0.5}\|\cdot\|_2}$ and~${\|\cdot\|_2 \leq \ol{{\pi}}^{0.5}\|\cdot\|_{\bds{\pi}}}$, where~${\ol{\pi}}$ and~${\ul{\pi}}$ are the maximum and minimum elements of~$\bpi$, respectively, while~${\mn{\ul B}_{\bds{\pi}} \!=\! \mn{\ul B^\infty}_{\bds{\pi}} = \mn{I_n - \ul B^\infty}_{\bds{\pi}} = 1}$.

\begin{assump}[Smooth and strongly convex cost functions] \label{smooth_convex} Each local cost~$f_{i}$ is~$\mu$-strongly convex and each component cost~$f_{i,j}$ is~$L$-smooth.
\end{assump}
We define the class of $\mu$-strongly convex and $L$-smooth functions as~$\mc S_{\mu,L}$. It can be verified that~${f_i\in\mc S_{\mu,L}}$, ${\forall i}$, and~${F\in\mc S_{\mu,L}}$; consequently,~$F$ has a global minimum that is denoted by~$\mb z^*\in\mbb R^p$. For any function in~$S_{\mu,L}$, we define its condition number as~$\kappa := \frac{L}{\mu}$. Note that~$\kappa \geq 1$.

\subsection{Auxiliary Results} \label{lti}
Before we proceed, we write~\PS~ in a vector-matrix format. To this aim, we define global vectors~$\mb x^k,\mb z^k\mb g^k, \mb w^k$, all in~$\mbb R^{pn}$, and~$\mb y^k\in\mbb R^n$, i.e.,
\begin{align*}
\mb{x}^k &:= \left[ {\begin{array}{c}
    \mb{x}^{k}_1\\
    \vdots\\
    \mb{x}^{k}_n
  \end{array} } \right], \quad \mb{z}^k := \left[ {\begin{array}{c}
    \mb{z}^{k}_1\\
    \vdots\\
    \mb{z}^{k}_n
  \end{array} } \right], \quad \mb{g}^k := \left[ {\begin{array}{c}
    \mb{g}^{k}_1\\
    \vdots\\
    \mb{g}^{k}_n
  \end{array} } \right], \quad \mb{w}^k := \left[ {\begin{array}{c}
    \mb{w}^{k}_1\\
    \vdots\\
    \mb{w}^{k}_n
  \end{array} } \right], \quad \mb{y}^k := \left[ {\begin{array}{c}
    y^{k}_1\\
    \vdots\\
    y^{k}_n,
  \end{array} } \right],
\end{align*}
and global matrices as~$B:=\ul B\otimes I_p$ and~$Y_k:=\mbox{diag}(\mb y^k)\otimes I_p$, both in~$\mbb R^{pn\times pn}$.~$\PS$ in Algorithm~\ref{algo} can now be equivalently written as
\begin{subequations}\label{sys_vec}
\begin{align}
\mb{x}^{k+1} &= B \mb{x}^{k} - \alpha  \cdot \mb{w}^{k},\\\label{y_eq}
\mb{y}^{k+1} &= \ul{B} \mb{y}^{k},\\
\mb{z}^{k+1} &= Y^{-1}_k \mb{x}^{k+1},\\
\mb{w}^{k+1} &= B \mb{w}^{k} + \mb{g}^{k+1} - \mb{g}^{k}.
\end{align}
\end{subequations}
To aid the analysis, we define four error quantities:
\begin{inparaenum}[(i)]
    \item Network agreement error:~$\mbb{E} \| \mb{x}^k - B^\infty \mb{x}^k \|^2 $;
    \item Optimality gap:~${\mbb{E} \| \ol{\mb{x}}^k - \mb{z}^* \|^2}$, where~${\ol{\mb x}^k:=\tfrac{1}{n}(\mb 1_n^\top \otimes I_p)\mb x^k}$;
    \item Mean auxiliary gap:~${\mbb{E} [ \mb{t}^k]}$, where ${\mb{t}^{k} := \sum_{i=1}^{n} (\frac{1}{m_i} \sum_{j=1}^{m_i} \| \mb{v}^{k}_{i,j} - \mb{z}^* \|_2^2 )}$;
    \item Gradient tracking error:~${\mbb{E} \| \mb{w}^k - B^\infty \mb{w}^k \|^2}$.
\end{inparaenum} 
In the following, Lemma~\ref{main_lem} establishes a relationship between these errors with the help of Lemma~\ref{y_lem}.
\begin{lem}\cite{GP_neidch}\label{y_lem}
Consider Assumption~$\ref{col_stoc_w}$~and let~${Y^\infty:=\lim_{k\ra\infty} Y_k}$, then~${\mn{Y_{k} - Y^\infty}_2 \leq T \lambda^k},\forall k$, where~${T := \sqrt{h} \|\mb{1}_n - n \bds{\pi}\|_2}$ and~${h:=\ol{\pi}/\ul{\pi}}>1$.
\end{lem}
Lemma~\ref{y_lem} quantifies the convergence rate of~\eqref{y_eq} and is only meaningful for column stochastic weights. When the weights are doubly stochastic, as in undirected graphs, we have~${T=0}$ since~${\bpi=\tfrac{1}{n}\mb 1_n}$.

\begin{lem}\label{main_lem}
Consider~\PS~described in~\eqref{sys_vec}~under Assumptions~\ref{col_stoc_w},~\ref{smooth_convex}, and let ${y := \sup_k \mn{Y_k}_2}$, ${y_- := \sup_k \mn{Y_k^{-1}}_2}$, and for all~$k\geq0$, define~${\mb{u}^k, \mb{s}^k \in \mbb{R}^4}$ and~$G_\alpha, H_k \in \mbb{R}^{{4 \times 4}}$ as
\begin{align} 
\label{lti}
\begin{array}{ll}
\mb{u}^{k} :=
  \left[ {\begin{array}{c}
   \mbb{E}[\|\mb{x}^{k} - B^\infty \mb{x}^{k} \|^2 _{\bpi}] \\
   \mbb{E}[n \|\ol{\mb{x}}^{k} - \mb{z}^* \|_2 ^2] \\
   \mbb{E}[\mb{t}^{k}] \\
    \mbb{E}[L^{-2}\|\mb{w}^{k} - B^\infty \mb{w}^{k} \|^2 _{\bpi}]
  \end{array} } \right], \quad
  &G_\alpha := \left[ {\begin{array}{c c c c}
    \frac{1+\lambda^2}{2}  &0 &0 & \frac{2 \alpha^2 L^2}{1-\lambda^2}  \\
    \frac{2 \alpha L^2 \psi {\ol{\pi}}}{\mu} & 1- \frac{\alpha \mu}{2} & \frac{2 \alpha^2 L^2}{n} &0 \\
    \frac{2 \psi {\ol{\pi}}}{m} & \frac{2}{m} &1-\frac{1}{M} &0 \\
    \frac{188 \psi}{1-\lambda^2} &\frac{169\ul{\pi}^{-1}}{1-\lambda^2} &\frac{38\ul{\pi}^{-1}}{1-\lambda^2} &\frac{3+\lambda^2}{4}
  \end{array} } \right] , \\
  \mb{s}^{k} := \left[ {\begin{array}{c}
    \mbb{E} [\|\mb{x}^{k} \|^2_2]\\
    0\\
    0\\
    0
  \end{array} } \right], \quad
  &H_k := \left[ {\begin{array}{c c c c}
    0 &0 &0 &0 \\
    \frac{2 \alpha L^2 \psi}{\mu} &0 &0 &0 \\
    \frac{2 \psi}{m} &0 &0 &0 \\
    \tfrac{188\psi^2}{1-\lambda^2} &0 &0 &0
  \end{array} } \right]T\lambda^k;
\end{array}
\end{align}
where~$m:=\min_i m_i,M:=\max_i m_i$, and~$\psi:=yy_-^2(1+T) h$ is defined as the directivity constant. For the stepsize~$0 < \alpha \leq \frac{1-\lambda^2}{28 L \kappa \psi}$, we have
\begin{align} \label{sys_eq}
    \mb{u}^{k} &\leq G_\alpha \mb{u}^{k-1} + H_{k-1} \mb{s}^{k-1}.
\end{align}
\end{lem}

The proof of the above lemma is available in the supplementary material. We note that the constant~$\psi\geq1$ can be interpreted as a directivity constant, i.e., for undirected graphs~$\psi=1$. It is of interest to observe that the LTI system described in~\eqref{lti} reduces to the one derived in~\cite{GTVR} for undirected graphs, where~$\psi=1$ and~$T=0$. Our strategy to establish linear convergence of~$\PS$ is to first show that~$\mb u^k$ converges linearly to a zero vector, which subsequently leads to~$\mb z^k_i\ra\mb z^*$ linearly at each node~$i$. We establish the linear decay of~$G_\alpha$ in the following lemma. 

\begin{lem} \label{G_lem}
Consider~\PS~in~\eqref{sys_vec}~under Assumptions~\ref{col_stoc_w} and~\ref{smooth_convex}. If the stepsize~$\alpha$ is such that~${\alpha\in (0, \ol{\alpha})}$, where ${\ol{\alpha} := \min \left \{ \frac{1}{5 M \mu}, \frac{m}{M} \frac{(1-\lambda)^2}{400 L \kappa \psi } \right \}}$, then
\[\rho(G_\alpha) \leq \mn{G_\alpha }_\infty^{\bds\delta} \leq \gamma := 1 - \min \left \{ \frac{1}{20 M}, \frac{m}{1600 M }\frac{(1-\lambda)^2}{\kappa^2 \psi} \right \} <1,
\]
where~$\rho(\cdot)$ denotes the spectral radius of a matrix and~$\mn{\cdot}_\infty^{\bds\delta}$ is the matrix norm induced by the weighted max-norm~$\|\cdot\|_\infty^{\bds\delta}$, with respect to some positive vector~${\boldsymbol \delta}$.
\end{lem}
The proof can be found in the supplementary material. It can be verified that~$\lambda\leq\rho(G_\alpha)$ that leads to the fact that~$H_k$ in~\eqref{sys_eq} decays faster than~$G_\alpha$. We next characterize the convergence of~$\|\mb u^k\|_2$ to zero. 

\begin{lem}
Consider~\PS~in~\eqref{sys_vec}~under Assumptions~\ref{col_stoc_w} and~\ref{smooth_convex}. For~${\alpha\in (0, \ol{\alpha})}$, we have that~$\|\mb u^k\|_2$ converges to zero linearly at~$\mc O((\gamma+\xi)^k)$, where~$\lambda<1$ and~$\xi>0$ is arbitrarily small.
\end{lem}
\begin{proof}
By expanding~\eqref{sys_eq} and taking the norm on both sides, we have 
\begin{align*}
\|\mb u^k\|_2 &\leq \|G_\alpha^k \mb u^0\|_2 + \textstyle\sum_{r=0}^{k-1} \|G_\alpha ^{k-r-1}H_r~\mb s^{r}\|_2,\\
&\leq \big(\Gamma_1 + \Gamma_2\textstyle\sum_{r=0}^{k-1} \|\mb s^{r}\|_2\big)\gamma^{k},
\end{align*}
for some constants~$\Gamma_1>0$ and~$\Gamma_2>0$ and we have used that~$\lambda\leq\rho(G_\alpha)\leq\gamma$. It can be shown that
\begin{align*}
\|\mb{s}^r\|_2 \leq  6(y^2 + \ol{\pi})\|\mb{u}^r\|_{2} + 3 y^2 n \|\mb{z}^*\|^2_2, 
\end{align*}
which leads to (after using~$b:=6\Gamma_2(y^2 + \ol{\pi})$ and~$c:=3\Gamma_2 y^2 n \|\mb{z}^*\|^2_2$)
\begin{align*}
\|\mb u^k\|_2 \leq \big(\Gamma_1 + kc +b\textstyle\sum_{r=0}^{k-1} \|\mb{u}^r\|_{2} \big) \gamma^{k}.
\end{align*}
Let~$v_k := \sum_{r=0}^{k-1} \|\mb{u}^r\|_2, c_k := (\Gamma_1+kc) \gamma^{k}$, and~$b_k := b \gamma^{k}$. Then, we have
\begin{align} \label{u_eq1}
\|\mb{u}^{k} \|_2 &= v_{k+1} - v_{k}\leq (\Gamma_1 + kc + b v_k) \gamma^{k} \iff v_{k+1} \leq (1 + b_k) v_k + c_k.
\end{align}
For non-negative sequences~$\{v_k\},\{b_k\}$, and~$\{c_k\}$, such that~${v_{k+1} \leq (1 + b_k) v_k + c_k}$, for all~$k$, and ${\sum_{k=0}^{\infty} b_k < \infty}$ and~${\sum_{k=0}^{\infty} c_k < \infty}$, we have from~\cite{polyak} that~$\{v_k\}$ converges and is therefore bounded. Hence,~${\forall \theta \in (\gamma, 1)}$, we can write~\eqref{u_eq1} as
\begin{align*}
    \lim_{k \rightarrow \infty} \frac{\|\mb{u}^k\|_2}{\theta^k} \leq \lim_{k \rightarrow \infty} \frac{(\Gamma_1 + kc + b v_k) \gamma^{k}}{\theta^k} = 0.
\end{align*}
In other words, there exist some~$\phi>0$ such that for all~$k$,
\begin{align} \label{u_eq2}
    \| \mb{u}^{k} \|_2 \leq \phi (\gamma + \xi)^{k},
\end{align}
where~$\xi>0$ is arbitrarily small. 
\end{proof}

With the help of the above lemma, we are now in a position to prove Theorem~\ref{th_main}.

\subsection{Proof of Theorem~\ref{th_main}}
Recall that~$\mb z_i^k$ is the estimate of~$\mb z^*$ at node~$i$ and iteration~$k$, and~$\mb z^k$ concatenates these local estimates in a vector in~$\mbb R^{pn}$. We have that
\begin{align*}
\mbb E[\|\mb{z}^k - (\mb{1}_n \otimes \mb{z}^*)\|^2_2] &= \mbb E[\|Y^{-1}_k \mb{x}^k - Y^{-1}_k Y^\infty (\mb{1}_n \otimes \ol{\mb{x}}^k) + Y^{-1}_k Y^\infty (\mb{1}_n \otimes \ol{\mb{x}}^k) \\
&~~~~- Y^{-1}_k Y^\infty (\mb{1}_n \otimes \mb{z}^*) + Y^{-1}_k Y^\infty (\mb{1}_n \otimes \mb{z}^*) - (\mb{1}_n \otimes \mb{z}^*) \|^2_2] \\
&\leq 3 y_-^2 \mbb E[\|\mb{x}^k - B^\infty \mb{x}^k \|_2^2] 
+ 3 n y_-^2 y^2 \mbb E[\|\ol{\mb{x}}^k - \mb{z}^* \|^2_2] + 3 n (y_- T \lambda^k)^2 \|\mb{z}^* \|_2^2\\
&\leq 6 y_-^2(y^2 + \ol{\pi}) \|\mb{u}^k \|_2 + 3 n y_-^2 T^2 \gamma^{k} \|\mb{z}^* \|_2^2 \\
&\leq 6 y_-^2(y^2 + \ol{\pi}) \phi (\gamma + \xi)^{k} + 3 n y_-^2 T^2 (\gamma+\xi)^{k} \|\mb{z}^* \|_2^2\\
&\leq \omega (\gamma + \xi)^{k},
\end{align*}
where we use~${Y^\infty(\mb 1_n\otimes \ol{\mb x}^k)=B^\infty\mb x^k}$ and~${\omega := 6 y_-^2(y^2 + \ol{\pi}) \phi  + 3 n y_-^2 T^2 \|\mb{z}^* \|_2^2}$. To reach an~$\epsilon$-accurate solution~$\mbb E[\|\mb{z}^{k} - (\mb{1}_n \otimes \mb{z}^*)\|^2_2] \leq\epsilon$, we thus need
\begin{align*}
\mbb E[\|\mb{z}^{k} - (\mb{1}_n \otimes \mb{z}^*)\|^2_2] \leq (1-(1-(\gamma+\xi)))^k~\omega\leq e^{-(1-(\gamma+\xi))k}~\omega\leq \epsilon,
\end{align*}
which leads to 
\begin{align*}
k\geq \max\left\{\tfrac{20M}{1-20M\xi},\tfrac{1600M\kappa^2 \psi}{m(1-\lambda)^2-1600M\kappa^2 \psi\xi}\right\}\ln~\tfrac{\omega}{\epsilon},
\end{align*}
iterations per node, where recall that~$\xi>0$ is arbitrarily small and the theorem follows.

\section{Numerical Experiments} \label{num_expt}
We now provide numerical experiments to compare the performance of~\PS~with related algorithms implemented over directed graphs with the help of MNIST and CIFAR-10 datasets, where the image sizes are~${28 \! \times \! 28}$~pixels and~${32 \! \times 32 \! \times \! 3}$~pixels, respectively. We study two classical learning techniques: logistic regression with a strongly convex regularizer and neural networks. 

\subsection{Logistic Regression (Strongly convex)} 
We first consider binary  classification of~${N=12,\!000}$~labelled images (taken from two classes in the MNIST and CIFAR-10 datasets) divided among~$n$ nodes with the help of logistic regression with a strongly convex regularizer. We compare~\SGP~(\DSGD~plus push-sum)~\cite{SGP_nedich},~\SA~(\SGP~plus gradient tracking)~\cite{saddopt}, and $\PS$; along with~\GP~\cite{opdirect_Tsianous} and~\ADD~\cite{add-opt,DIGing}, which are the deterministic counterparts of~\SGP~and~\SA, respectively. Note that~\ADD~(\GP~plus gradient tracking) converges linearly to the exact solution  since it uses the full batch data at each node. We plot the optimality gap~$F{(\ol{\mb z}^k) - F(\mb z^*)}$, where~${\ol{\mb z}^k:=\frac{1}{n}\sum_i\mb z_i^k}$, of each algorithm versus the number of epochs, where each epoch represents~$m_i$ gradient computations per node, i.e., one epoch is one iteration of~\GP~and~\ADD~and~$m_i$~iterations for the stochastic algorithms~\SGP,~\SA,~and~\PS.

Fig.~\ref{G1} compares the  algorithms over a ${n=16}$-node exponential graph, when the data is equally divided among the nodes, i.e.,~$m_i=750,\forall i$. This scenario models a controlled training setup over a well-structured communication network as, e.g., in data centers. Similarly, Fig.~\ref{G2} illustrates the performance over a ${n=500}$-node geometric graph when the data is divided arbitrarily among the nodes modeling, e.g., ad hoc edge computing networks. It can be verified that~\PS~converges linearly for smooth and strongly convex problems and is much faster, in terms of epochs, compared with its linearly-convergent non-stochastic counterpart~\ADD~over directed graphs.
\begin{figure}[!h]
\centering
\subfigure{\includegraphics[width=1.95in]{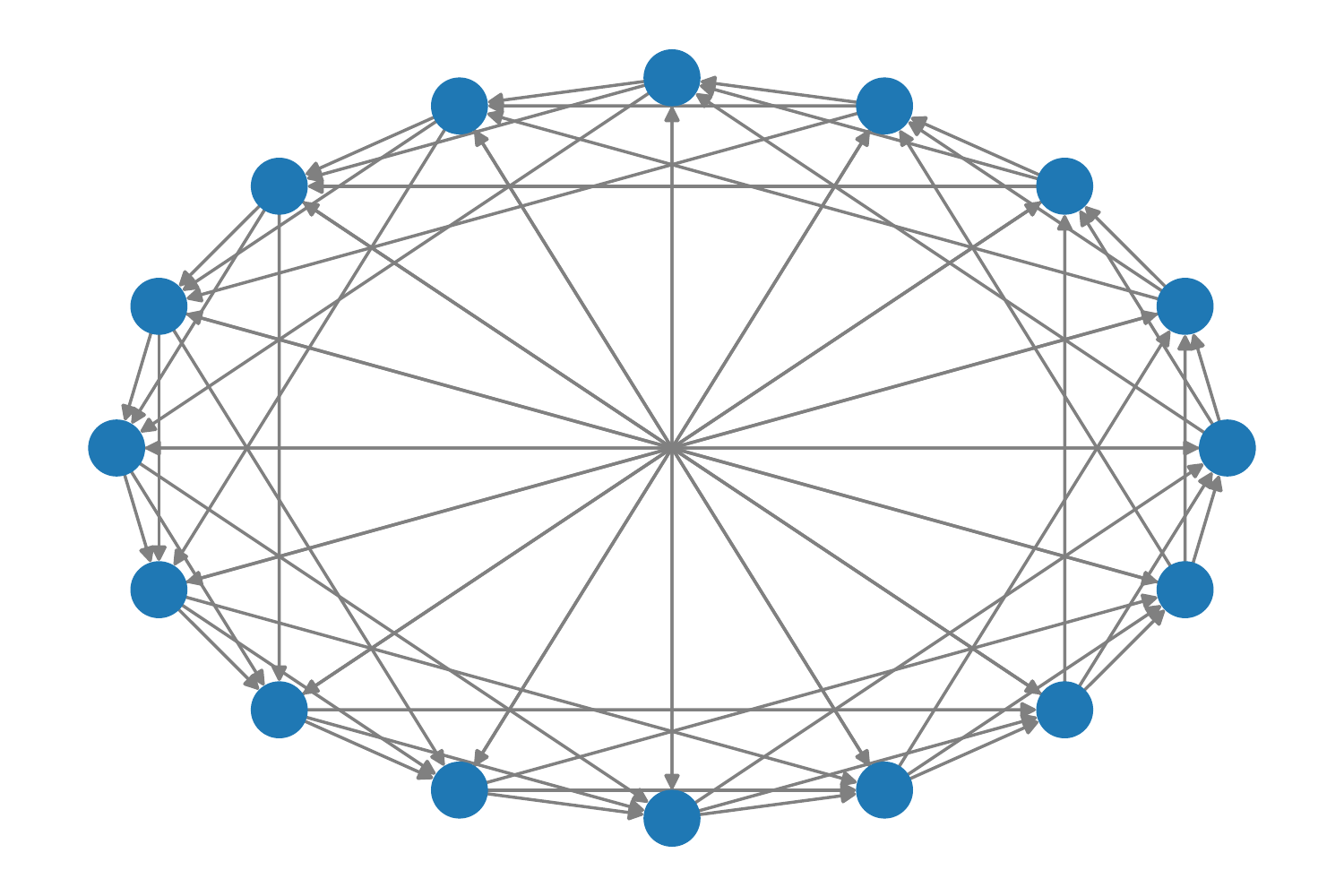}}
% \hspace{0.2cm}
\subfigure{\includegraphics[width=1.93in]{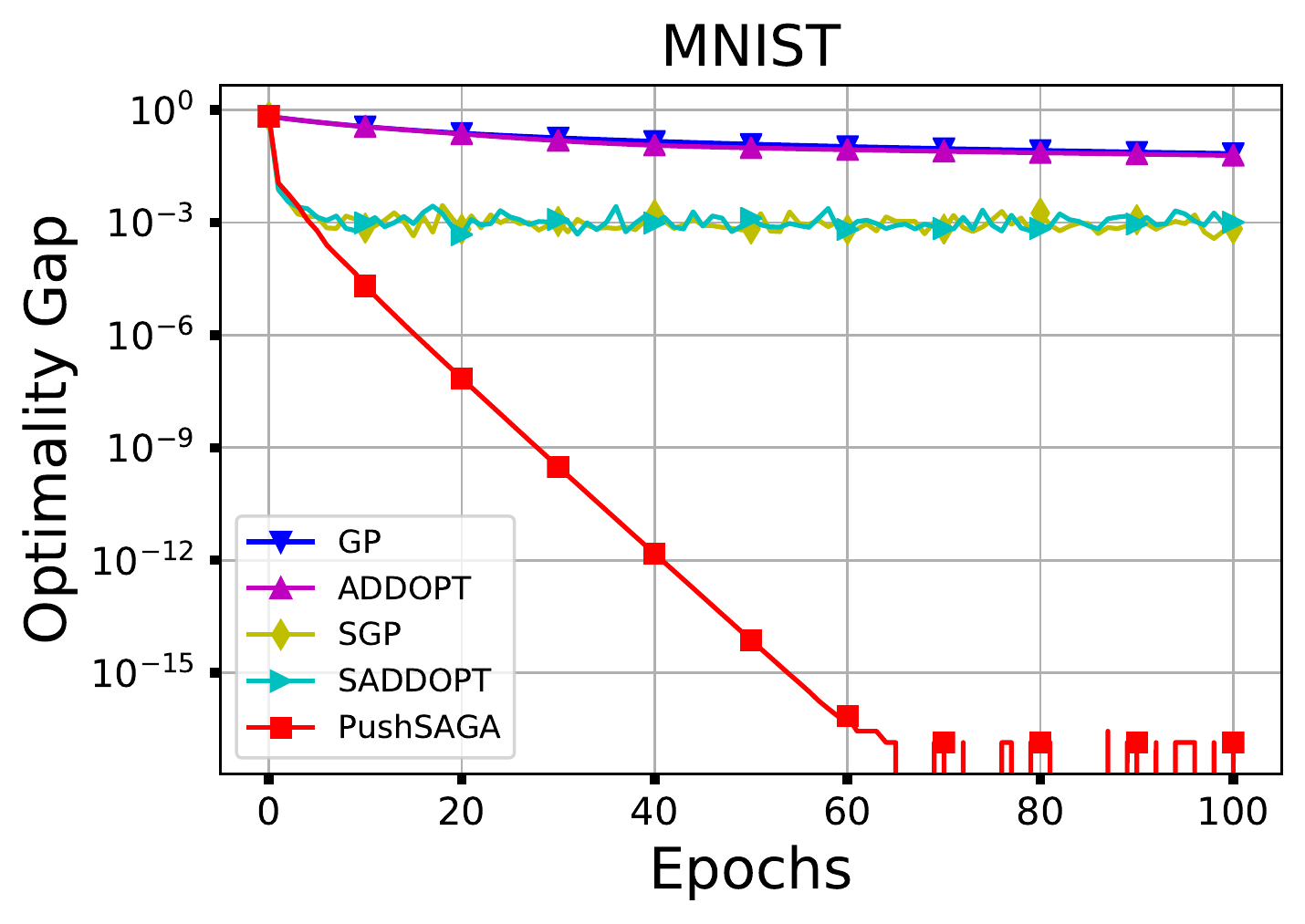}}
\hspace{0.2cm}
\subfigure{\includegraphics[width=1.93in]{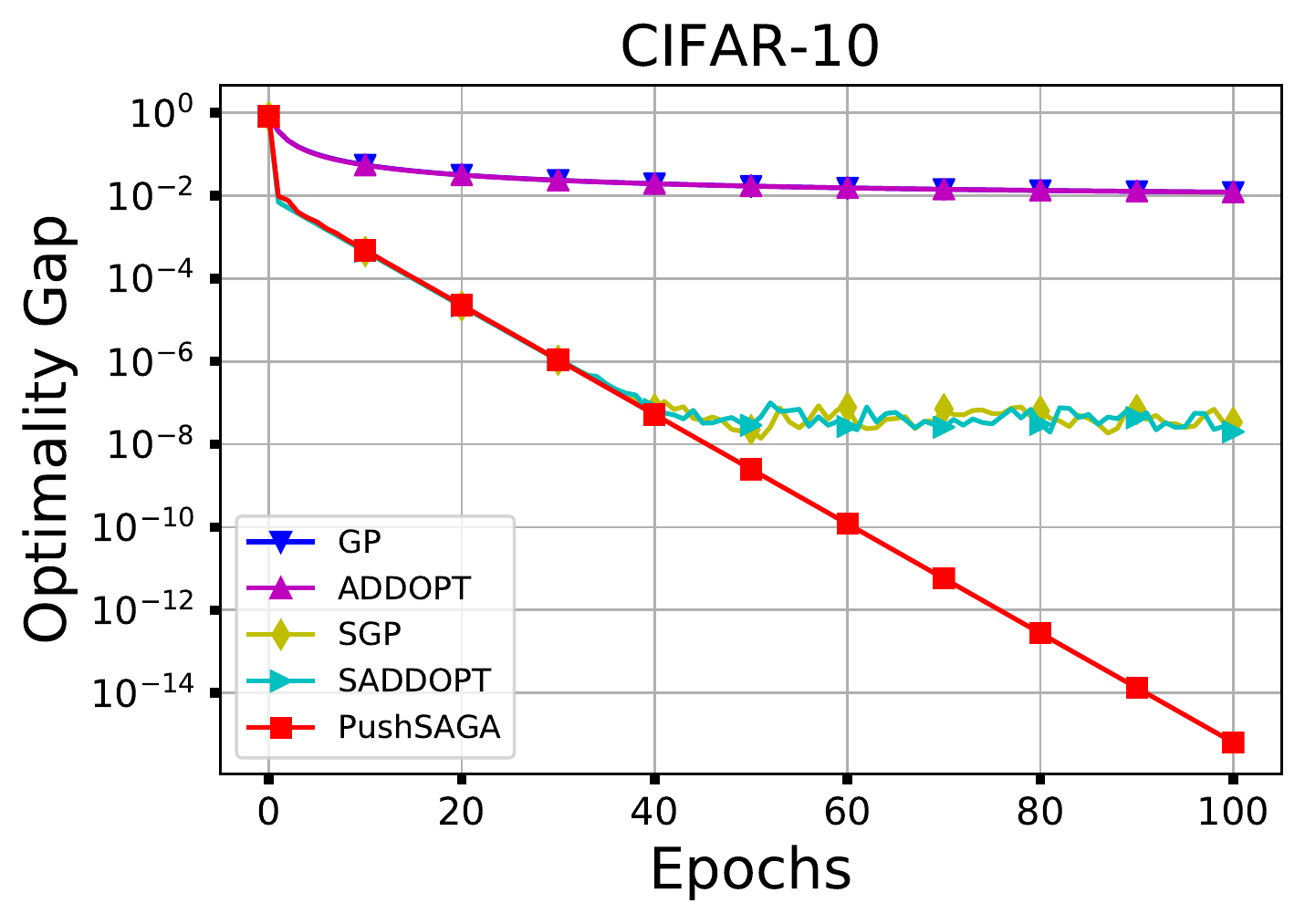}}
\caption{Performance comparison over a directed exponential graph with~$n=16$~nodes.}
\label{G1}
\end{figure}

\begin{figure}[!h]
\centering
\subfigure{\includegraphics[width=1.95in]{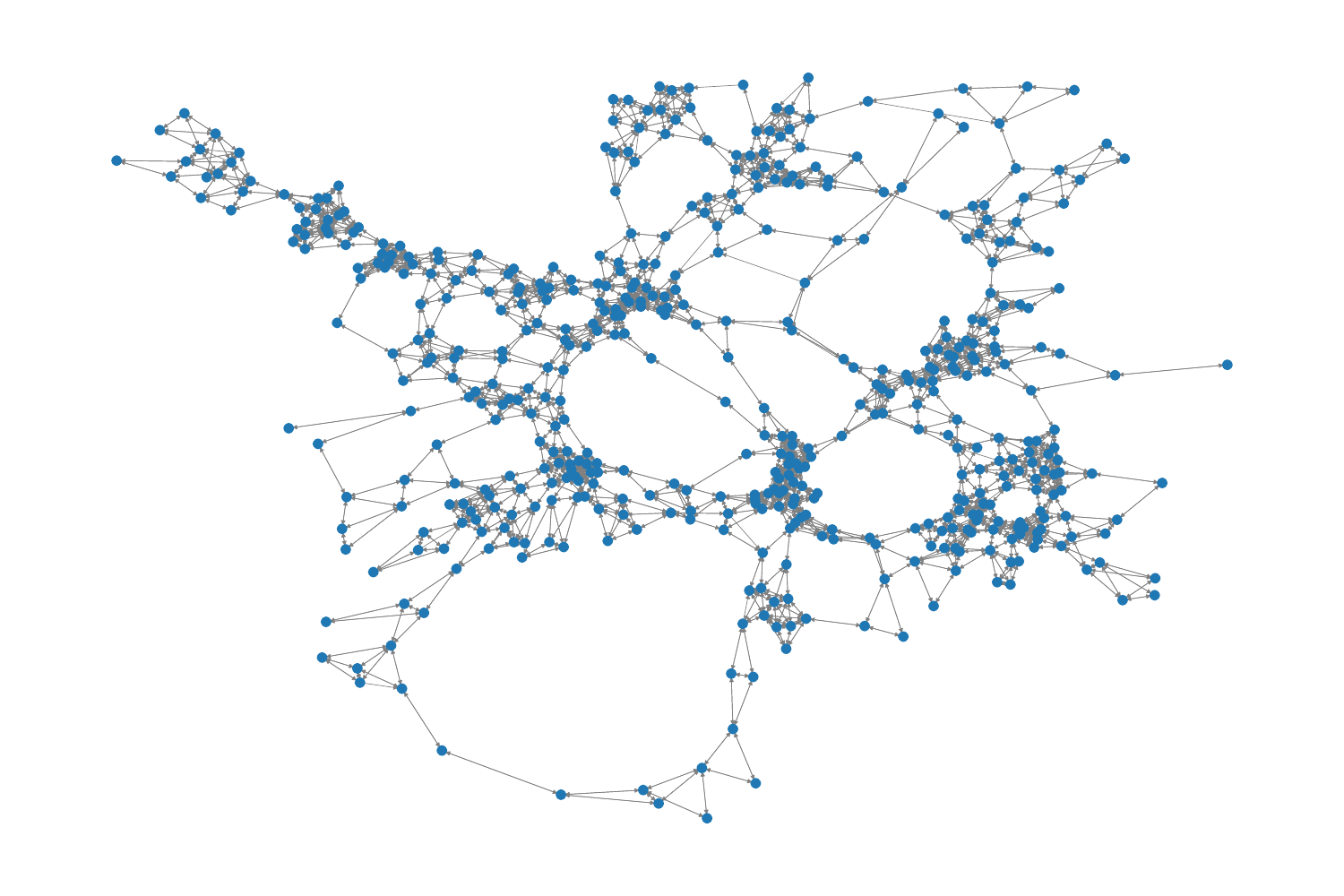}}
% \hspace{0.2cm}
\subfigure{\includegraphics[width=1.93in]{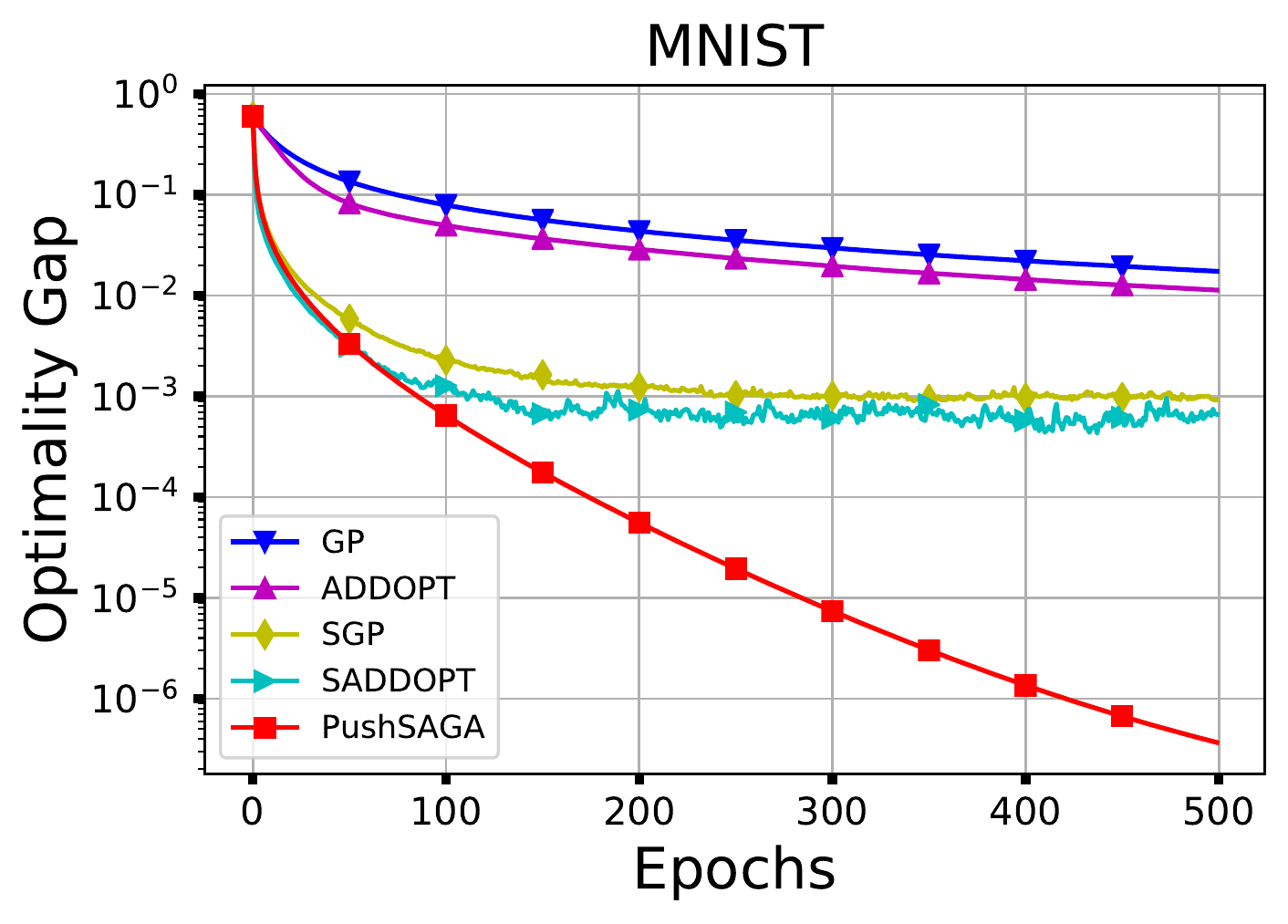}}
\hspace{0.2cm}
\subfigure{\includegraphics[width=1.93in]{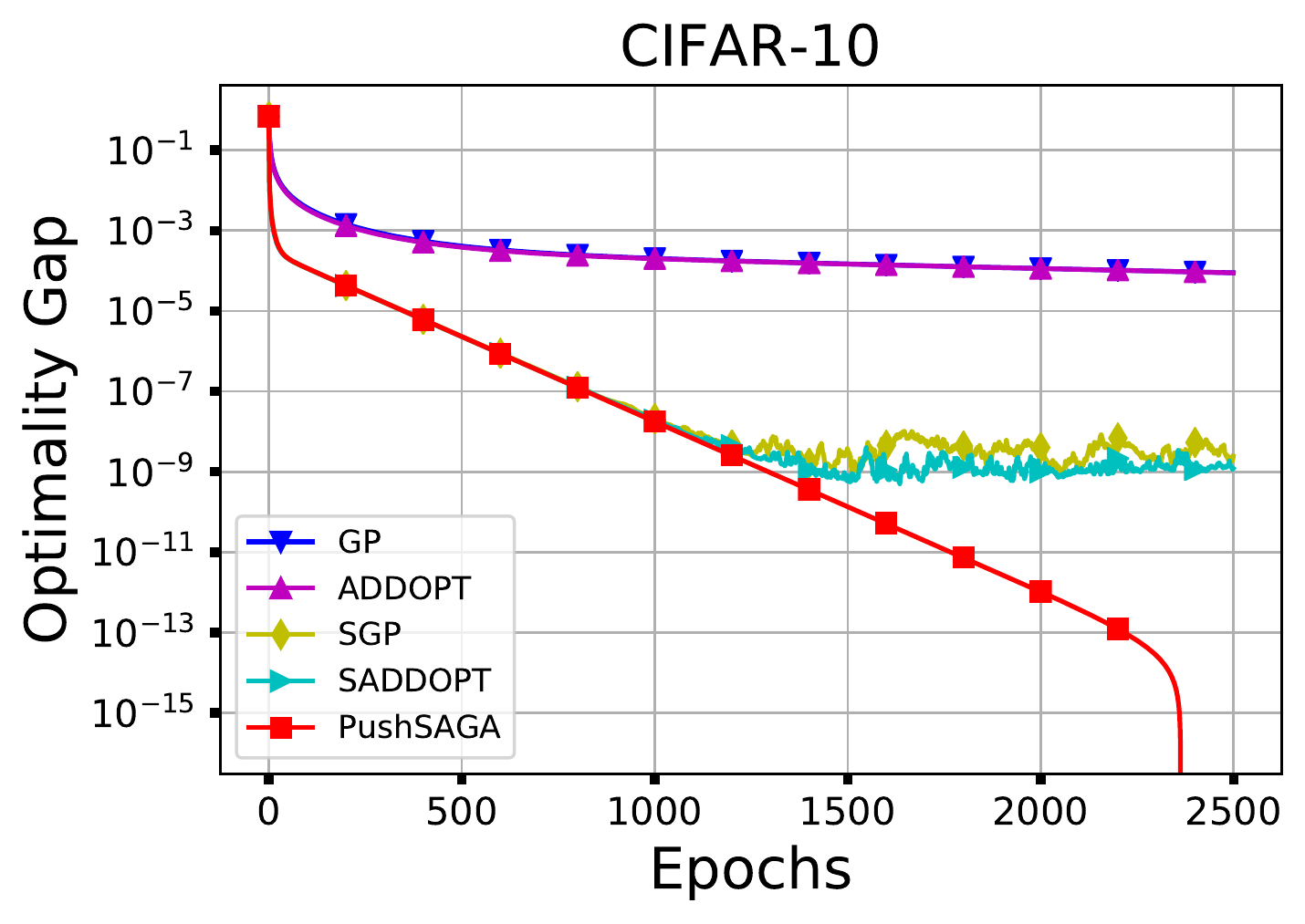}}
\caption{Performance comparison over a directed geometric graph with~$n=500$~nodes.}
\label{G2}
\end{figure}
\newpage
\textbf{Linear speed-up:} We next show the linear speed-up of~\PS~over its centralized counterpart~\SG. For illustration,~\SGP~and~\SA~are also compared with their centralized counterpart~\SGD. To this aim, we study binary classification based on the MNIST dataset over exponential graphs with~$n=4,8,16,32,$ and~$64$ nodes, and plot the ratio of the number of iterations of the centralized algorithm and its decentralized counterpart to reach a certain optimality gap~$\epsilon_1$. We choose~${\epsilon_1=10^{-15}}$ for~\PS~and~\SG, since they linearly converge to the exact solution, while~${\epsilon_1=10^{-3}}$ for
\SGP,~\SA, and~\SGD, since they linearly converge to an error ball (with a constant stepsize). Recall that one iteration involves one gradient computation in centralized~\SG~and~$n$ (parallel) gradient computations in~\PS. Fig.~\ref{speed} shows that the per node complexity of~\PS~is~$\mc O(n)$ times faster than~\SG;~\PS~thus acts effectively as a means for parallel computation.
\begin{figure}[!h]
\centering
\subfigure{\includegraphics[height=1in]{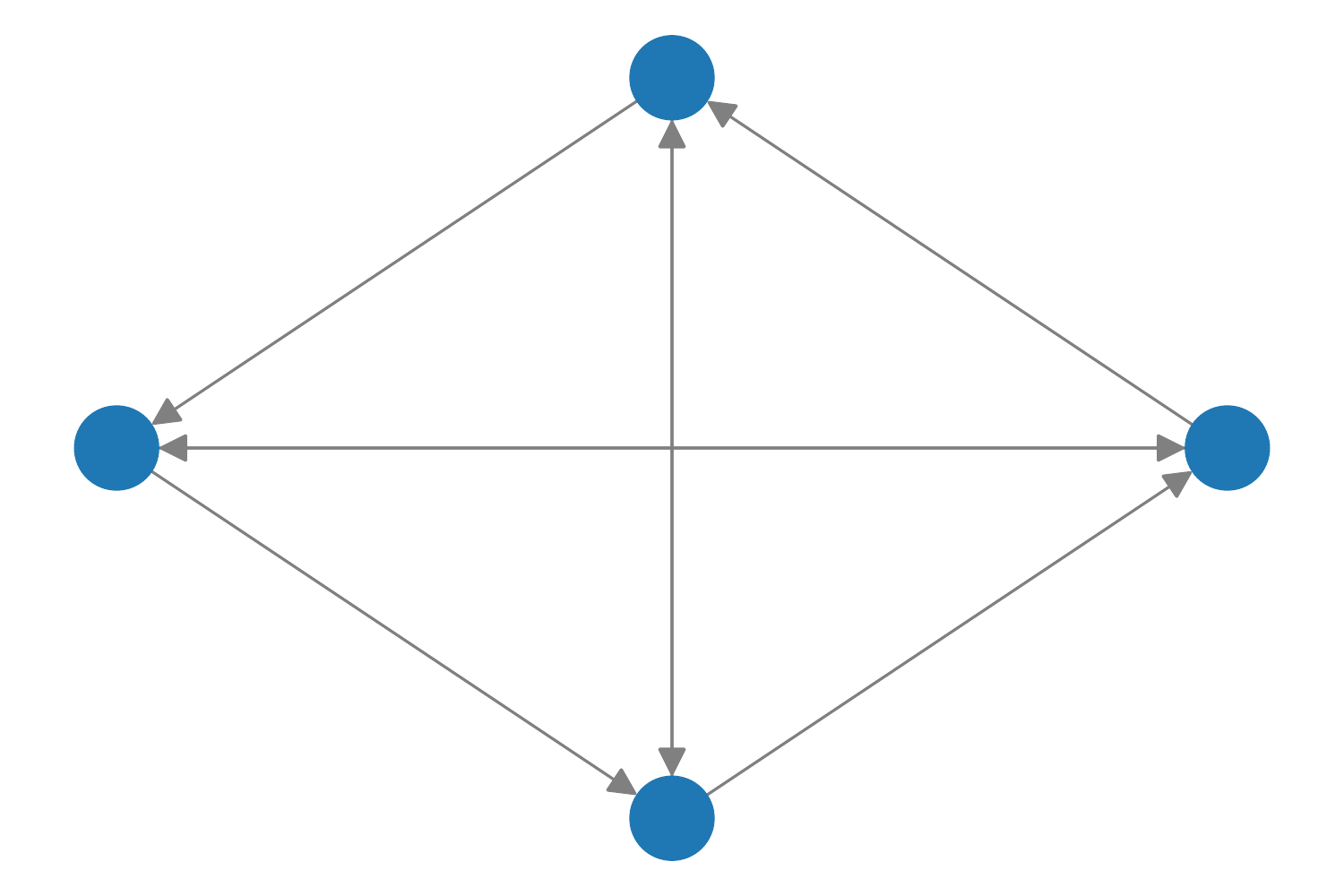}}
\subfigure{\includegraphics[height=1in]{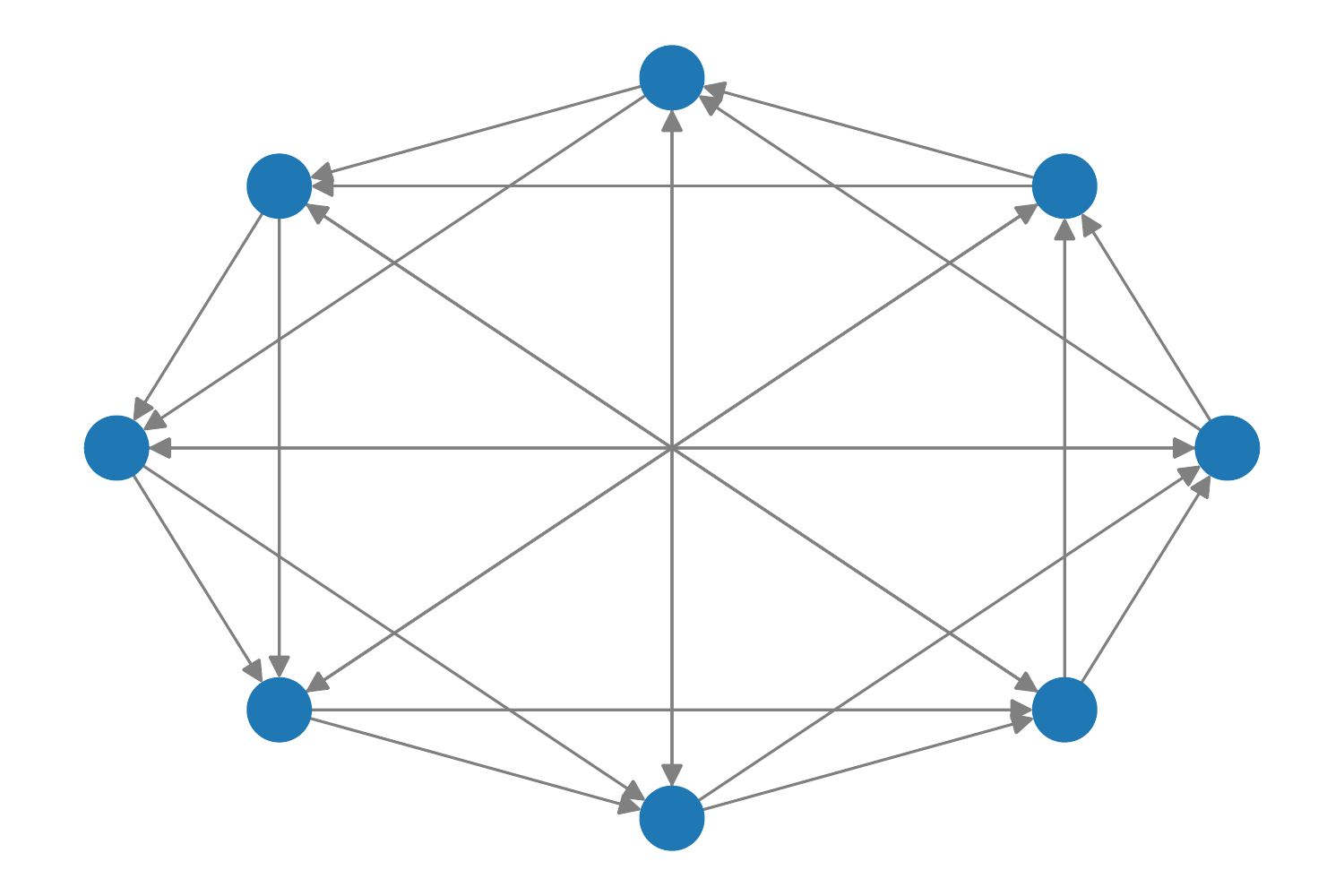}}
\subfigure{\includegraphics[height=1in]{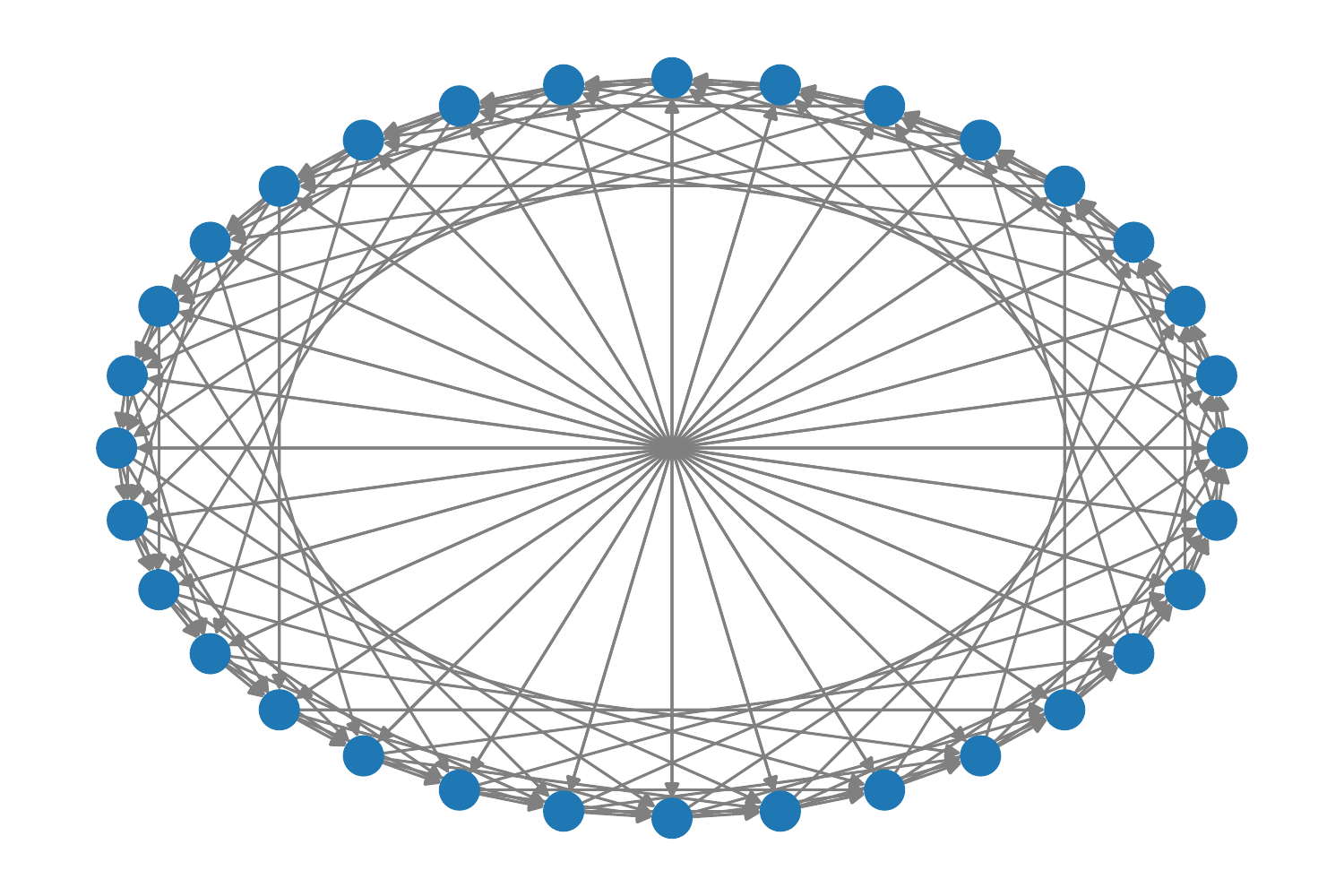}}
\subfigure{\includegraphics[height=1in]{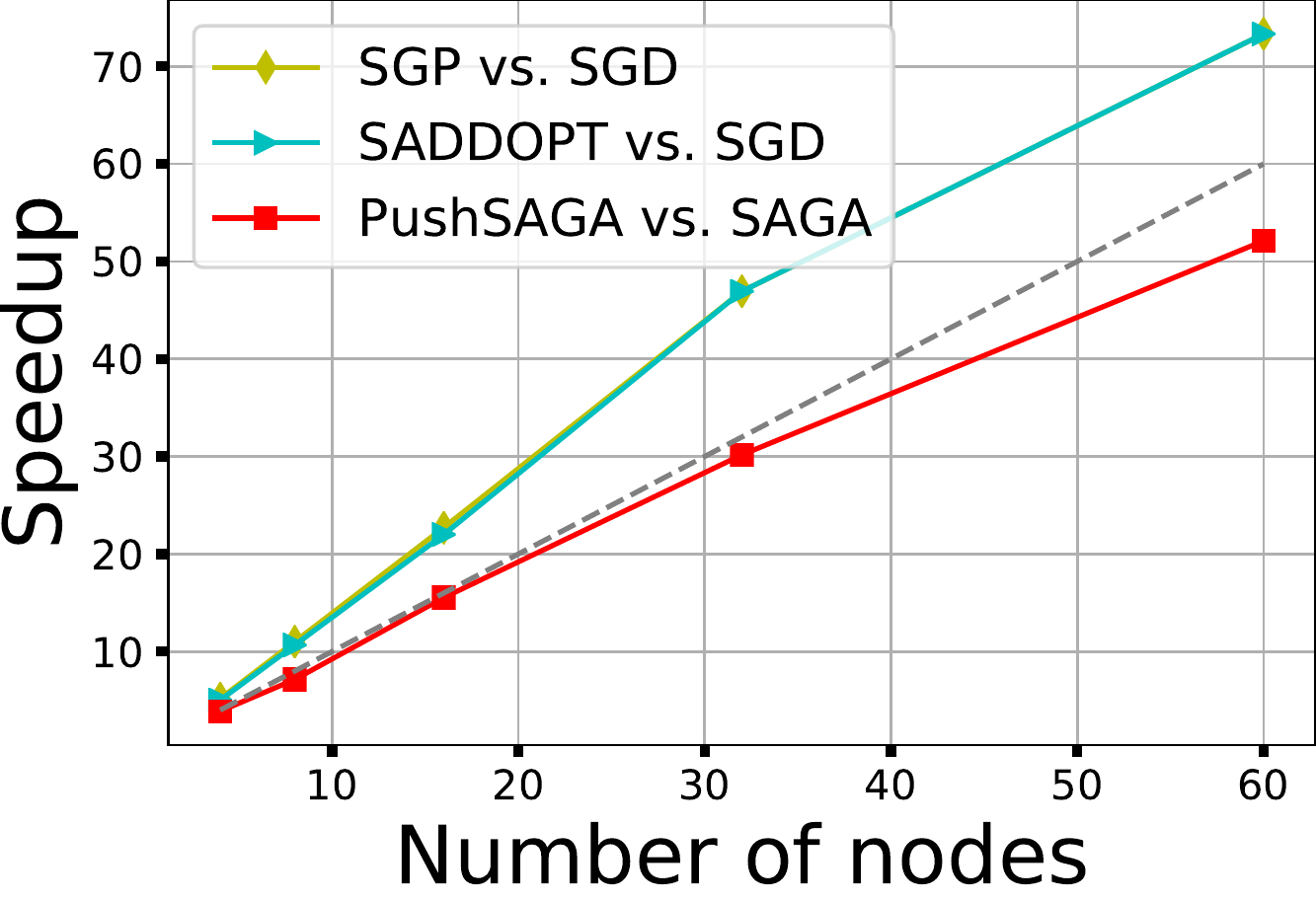}}
\caption{Exponential graphs with~$n=4,8,32$ nodes. The plot shows~\SGP~and~\SA~vs.~\SGD~to achieve an optimality gap of~${10^{-3}}$;~\PS~vs.~\SG~to achieve an optimality gap of~${10^{-15}}$.}
\label{speed}
\end{figure}

\textbf{Network independent convergence: }We now demonstrate the network-independent convergence behavior of~\PS~in the big-data regime, i.e., when~${M=m\gg\kappa^2\psi(1-\lambda)^{-2}}$. For this purpose, we choose the binary classification problem based on the MNIST dataset, with~${N=nm=12,\!000}$ total images equally divided over a network of~$n$ nodes, and keep~${\kappa\approx 1}$ with an appropriate choice of the regularizer. We start with the base directed cycle and generate additional subsequent graphs by adding random directed edges until the graph is almost complete. The family of graphs in an~$n$-node setup thus ranges from least-connectivity (a directed cycle with~${\lambda\ra1}$) to improved connectivity by the addition of edges. For each family of graphs (fixed~$n$), we plot the optimality gap of~\PS~in Fig.~\ref{nic}, for~${n=\{4,8,16,32\}}$ nodes that leads to~${m=\{{3,\!000}, {1,\!500}, 750, 375\}}$ data samples per node. We observe that as long as~$m$ is sufficiently larger than~${\psi(1-\lambda)^{-2}}$, the convergence of~\PS~is almost the same across all topologies generated by keeping a fixed number of nodes~$n$.~\PS~loses network-independence for the~${n=32}$-node network; this is because the big-data regime does not apply as it can be verified that~${m=375}$ and~$ m\approx\psi(1-\lambda)^{-2}$. 
\begin{figure}[!h]
\centering
\subfigure{\includegraphics[width=1.55in]{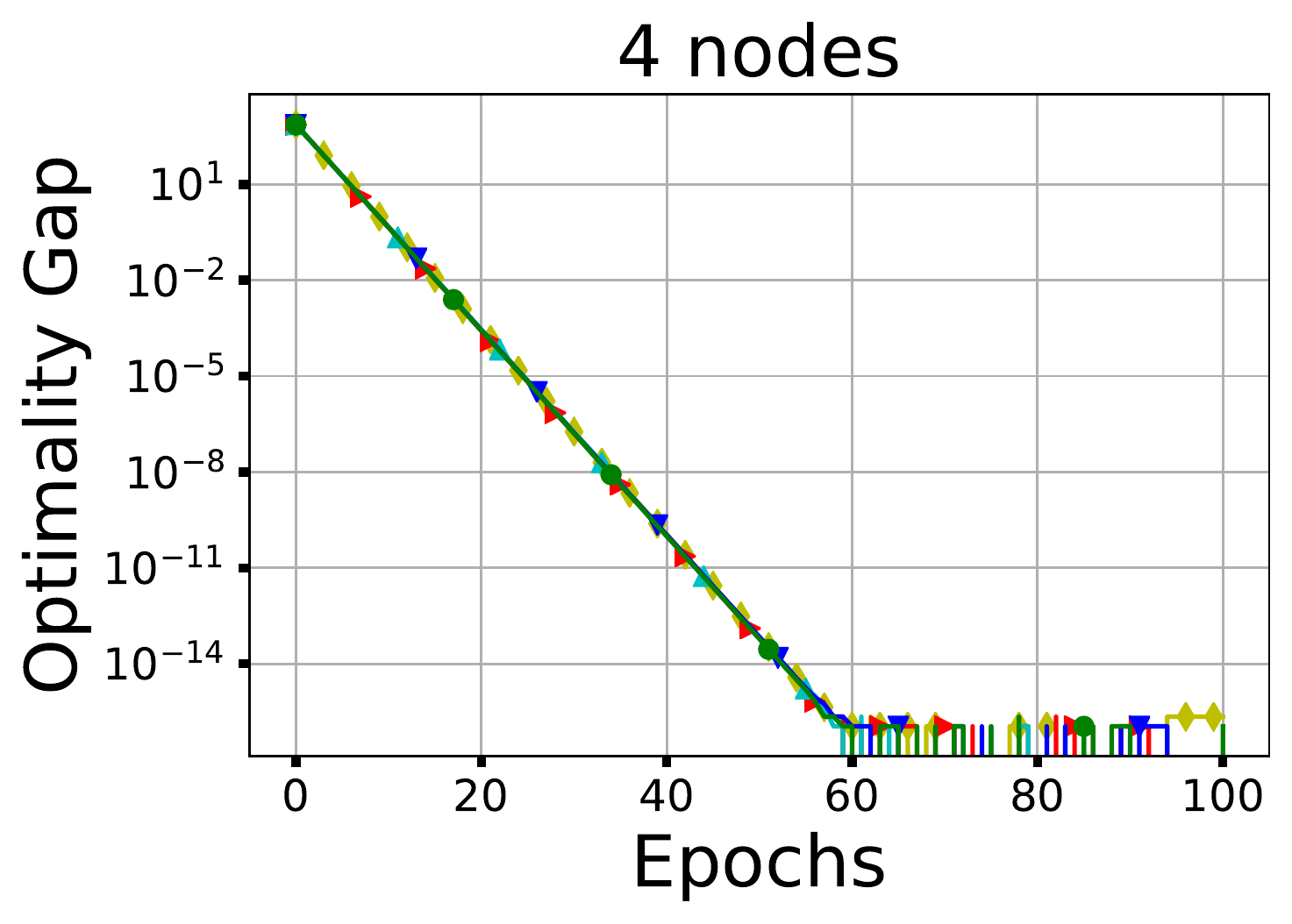}}
\subfigure{\includegraphics[width=1.55in]{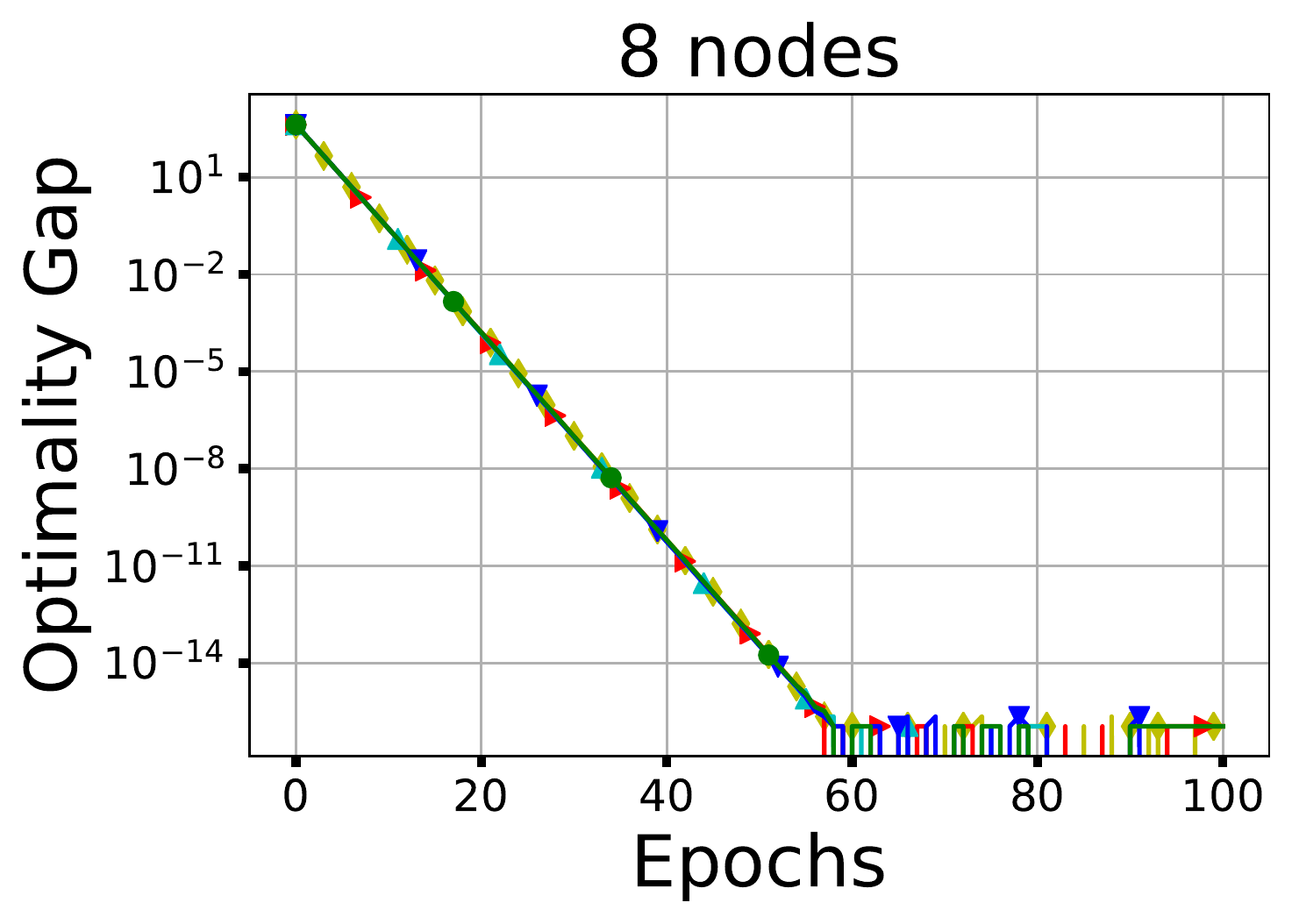}}
\subfigure{\includegraphics[width=1.55in]{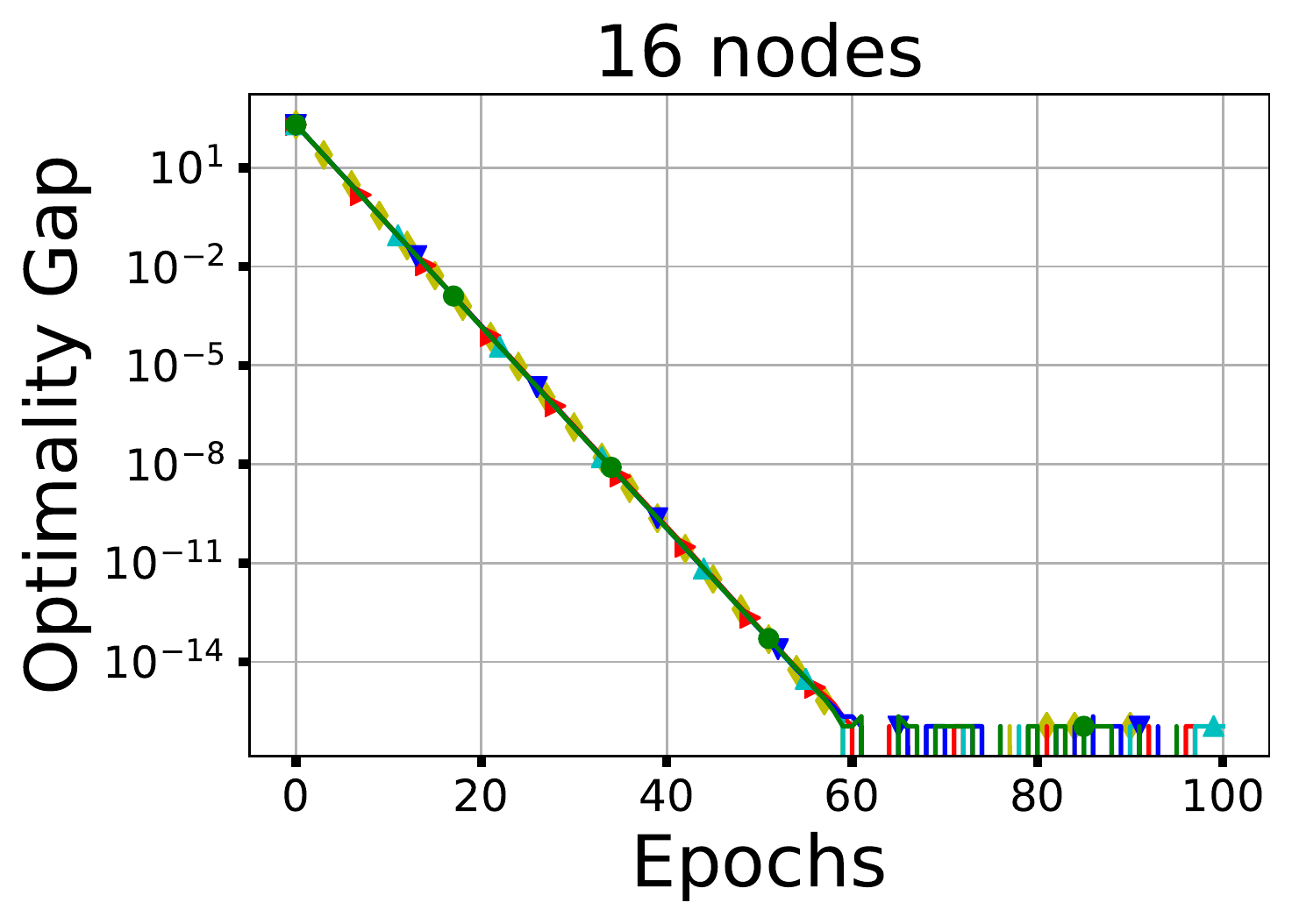}}
\subfigure{\includegraphics[width=1.55in]{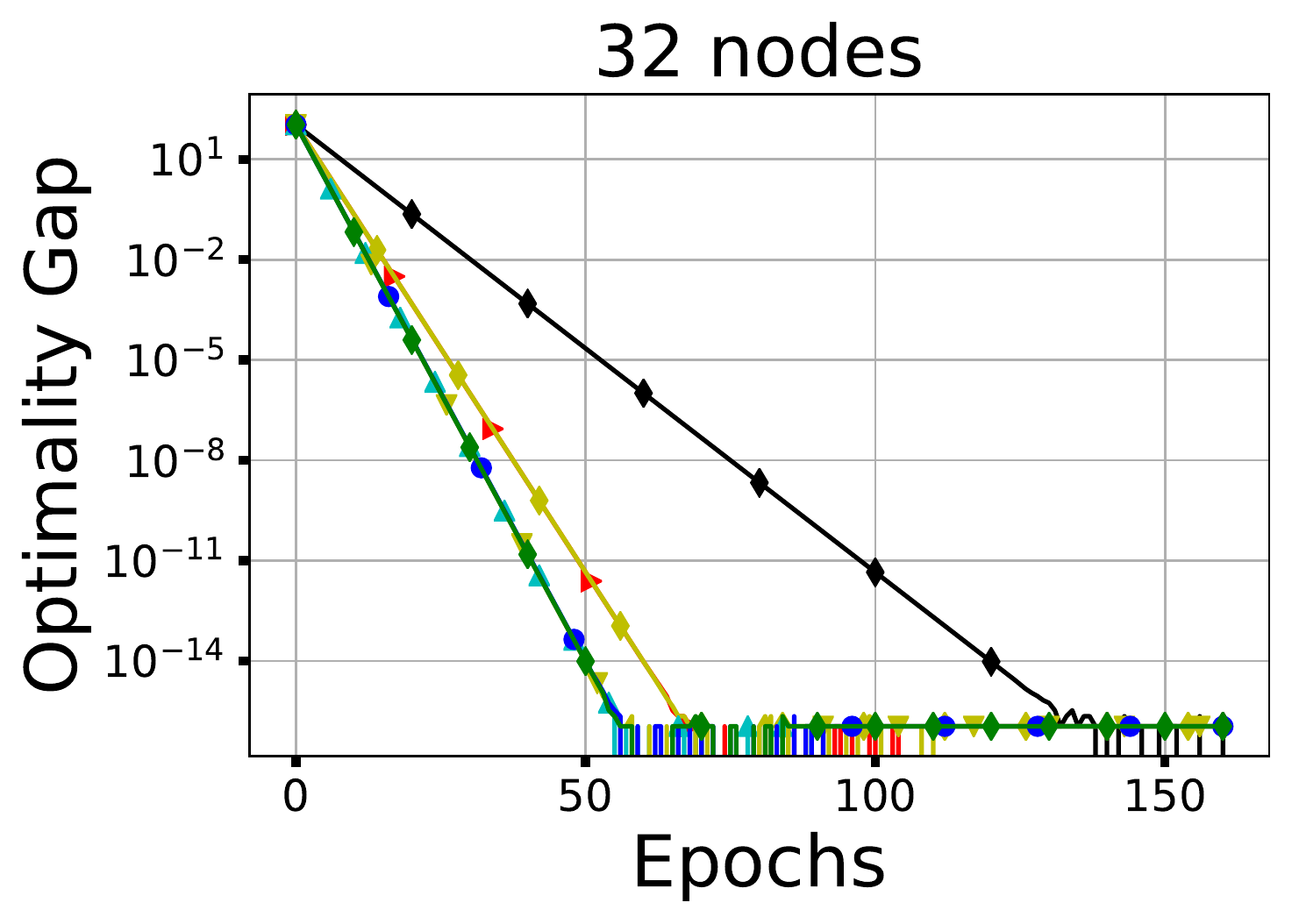}}
\caption{Network-independent convergence rates for~\PS~. Each figure spans different graphs of varying connectivity while keeping the number of nodes~$n$ fixed.}
\label{nic}
\end{figure}

\subsection{Neural Networks (Non-convex)} 
Finally, we compare the stochastic algorithms, i.e.,~\SGP,~\SA, and~\PS,~for training a neural network over directed graphs. For each node, we construct a custom two-layered neural network comprising of one fully-connected, hidden layer of~$64$ neurons. We consider a multi-class classification problem on the MNIST and CIFAR-10 datasets with~$10$ classes each. Both datasets consist of~$60,\!000$ images in total and~$6,\!000$ images per class. The data samples are evenly distributed among the nodes that communicate over the~$500$-node geometric graph of Fig.~\ref{G2}. We show the loss~$F(\ol{\mb{z}}^k)$ and the  test accuracy in Fig.~\ref{sim_NN}. It can be observed that~\PS~shows improved performance compared to other methods particularly over the CIFAR-10 dataset. 
\begin{figure}[!h]
\centering
\subfigure{\includegraphics[width=1.55in]{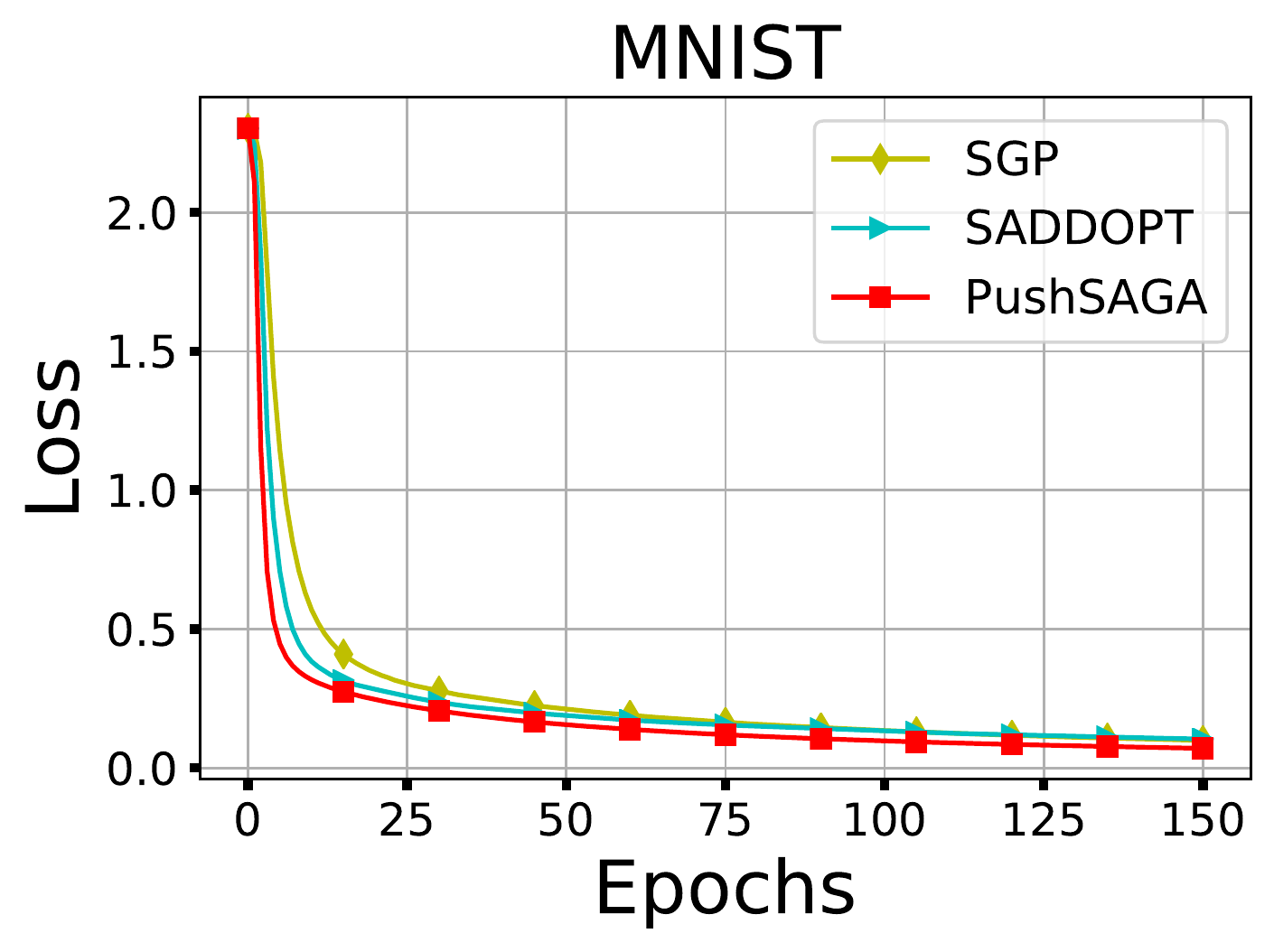}}
\subfigure{\includegraphics[width=1.55in]{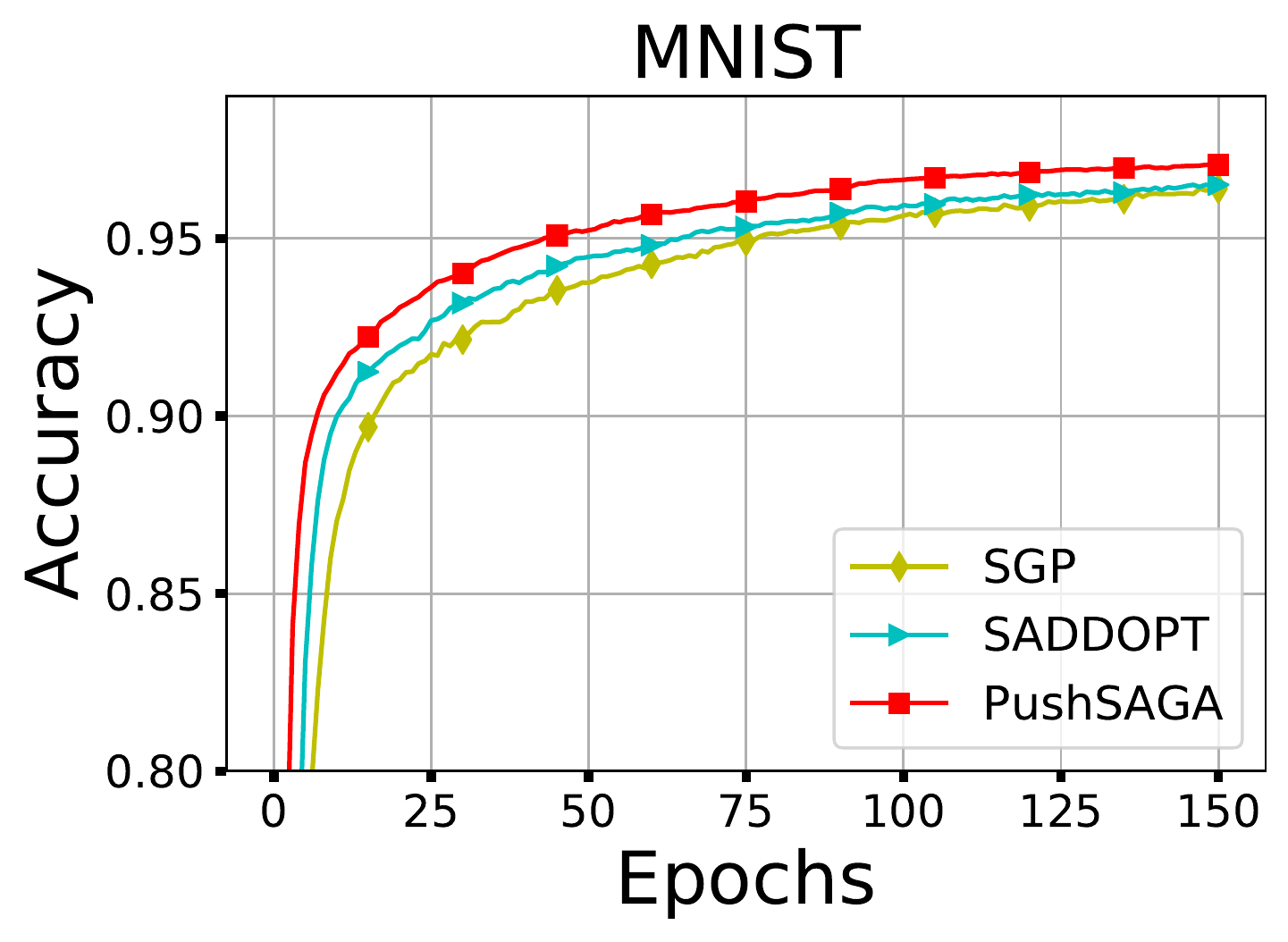}}
\subfigure{\includegraphics[width=1.55in]{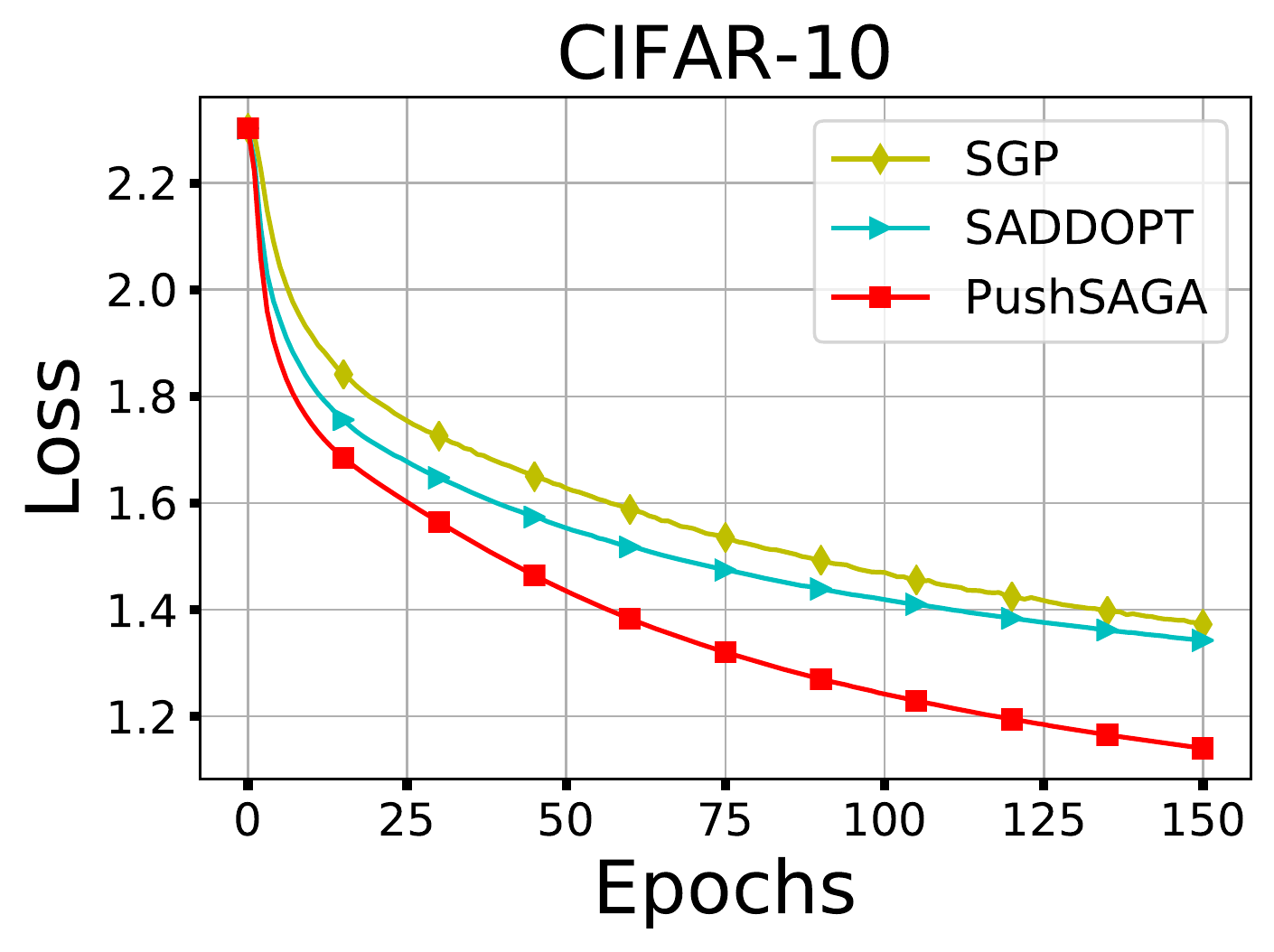}}
\subfigure{\includegraphics[width=1.55in]{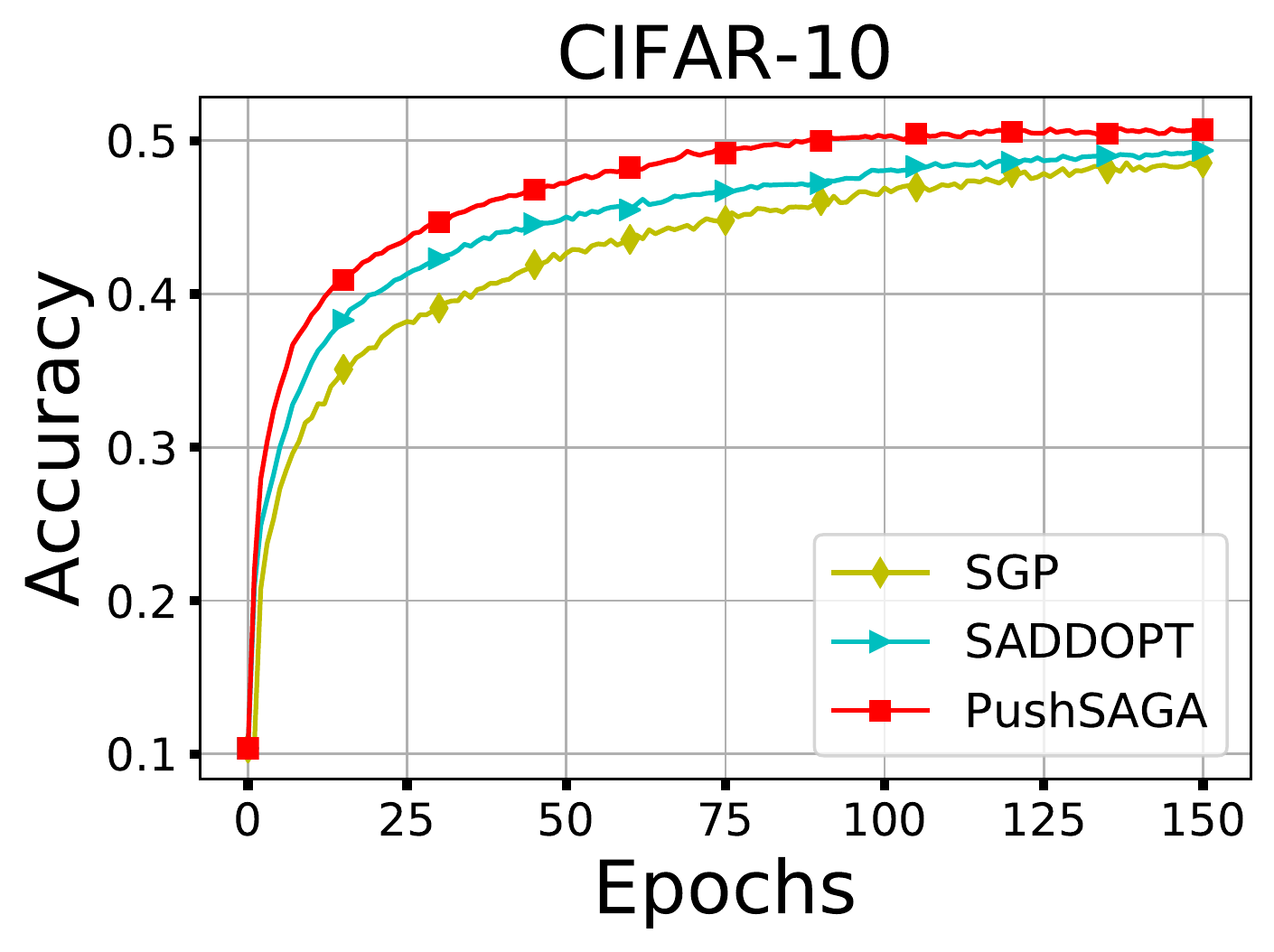}}
\caption{Training a two-layer neural network over a directed geometric graph with~$n=500$~nodes.}
\label{sim_NN}
\end{figure}
\bibliography{KHAN_allPUBS,Qureshi}
\bibliographystyle{unsrt}
\newpage
\appendices
\section{Proof of Lemma 2: LTI inequality of~\PS}
\begin{proof}
We first define some auxiliary variables:
\begin{align*}
     \ol{\mb{w}}^{k} &:= \frac{1}{n} (\mb{1}_n ^\top \otimes I_p) \mb{w}^{k}, \qquad\qquad\qquad\qquad \ol{\mb{h}}^{k} := \frac{1}{n} (\mb{1}_n ^\top \otimes I_p) \nabla \mb{f} (\mb{z}^{k}),\\ \ol{\mb{g}}^{k} &:= \frac{1}{n} (\mb{1}_n ^\top \otimes I_p) \mb{g}^{k}, \qquad\qquad\qquad\qquad \ol{\mb{p}}^{k} := \frac{1}{n} (\mb{1}_n ^\top \otimes I_p) \nabla \mb{f}(\mb 1_n \otimes \ol{\mb{x}}^{k}),\\
     \nabla\mb f(\mb z^k) &:= [\nabla {f}_1(\mb z_1^k)^\top, \cdots, \nabla {f}_n(\mb z_n^k)^\top]^\top,
\end{align*}
where~$\ol{\mb{w}}^{k},\ol{\mb{h}}^{k},\ol{\mb{g}}^{k},$ and~$\ol{\mb{p}}^{k}$, all in~$\mbb R^p$, are averages of the corresponding local vectors at the nodes, and~$\nabla\mb f(\mb z^k)\in\mbb R^{pn}$ stacks the local gradients available at the corresponding local iterates. 

It can be shown that~${\ol{\mb{w}}^{k}=\ol{\mb{g}}^{k}}, {\forall k}$, see e.g.,~\cite{add-opt}. We recall that~\PS~is a stochastic method and the randomness lies in the set of independent random variables~$\{s_i^k\}_{i=\{1, \cdots n\}}^{k\geq1}$. We denote~$\mc{F}^k$~as the history of dynamical system generated by~$\{s_i^a\}_{i=\{1, \cdots n\}}^{a\leq k-1}$. The derivation of~\eqref{sys_eq} is provided in the following four steps:

\paragraph{Step 1 -- Network agreement error:} We note that the first term describing the LTI system is the network agreement error and it can be expanded as,
\begin{align*}
   \|\mb{x}^{k+1} - B^\infty \mb{x}^{k+1} \|^2 _{\bpi} &= \|B \mb{x}^{k} - B^\infty \mb{x}^{k} - \alpha (\mb{w}^{k} - B^\infty \mb{w}^{k}) \|^2 _{\bpi}\\
   &\leq \left(1 + r \right) \lambda^2 \|\mb{x}^{k} - B^\infty \mb{x}^{k}\|^2 _{\bpi} + \left(1 + r^{-1} \right)\alpha^2  \| \mb{w}^{k} - B^\infty \mb{w}^{k} \|^2 _{\bpi},
\end{align*}
where we use that~${\|B \mb x^k - B^\infty \mb x^k \|_{\bpi} \leq \mn{B - B^\infty}_{\bpi}\|\mb x^k - B^\infty \mb x^k \|_{\bpi} = \lambda \|\mb x^k - B^\infty \mb x^k \|_{\bpi}}$, from Assumption~1,  and the Young's Inequality, i.e.,~${\forall \mb a, \mb b \in \mbb R^{pn}}$, and for~${r>0}$, 
\[
{\|\mb a + \mb b \|^2 \leq (1 + r)\|\mb a \|^2 + (1 + r^{-1})\|\mb b \|^2}.
\] 
Setting~${r = \frac{1-\lambda^2}{2 \lambda^2}}$~and~${r = 1}$~in the above inequality and taking full expectation leads to
\begin{align}\label{agg1}
   \mbb{E}\|\mb{x}^{k+1} - B^\infty \mb{x}^{k+1} \|^2 _{\bpi}
   &\leq \left(\frac{1+\lambda^2}{2}\right) \mbb{E} \|\mb{x}^{k} - B^\infty \mb{x}^{k} \|^2 _{\bpi} + \left(\frac{2 \alpha^2}{1-\lambda^2} \right) \mbb{E}\| \mb{w}^{k} -B^\infty \mb{w}^{k} \|^2 _{\bpi}, \\\label{agg2}
   \mbb{E}\|\mb{x}^{k+1} - B^\infty \mb{x}^{k+1} \|^2 _{\bpi}
   &\leq 2 \mbb{E}\|\mb{x}^{k} - B^\infty \mb{x}^{k} \|^2 _{\bpi} + 2 \alpha^2 \mbb{E}\| \mb{w}^{k} -B^\infty \mb{w}^{k} \|^2 _{\bpi},
\end{align}
where~\eqref{agg1} is used in the LTI system~\eqref{sys_eq}, and~\eqref{agg2} is helpful for the analysis. 
\paragraph{Step 2 -- Mean gap of the auxiliary variables:} We expand as follows:~${\forall k \geq 1}$ and ${\forall i}$,
\begin{align} \nonumber
\mbb{E}[\mb{t}^{k+1}_{i} | \mc{F}^{k}] &\leq \frac{1}{m_i} \sum_{j=1}^{m_i} \mbb{E} \left[ \left\| \mb{v}^{k+1}_{i,j} - \mb{z}^* \right\|_2^2 \Big| \mc{F}^{k} \right] \\ \nonumber
&= \frac{1}{m_i} \sum_{j=1}^{m_i} \left[ \left( 1 - \frac{1}{m_i} \right) \left\| \mb{v}^{k}_{i,j} - \mb{z}^* \right\|_2^2 + \frac{1}{m_i} \left\| \mb{z}^{k}_{i} - \mb{z}^* \right\|_2^2 \right] \\ \nonumber
&= \left( 1 - \frac{1}{m_i} \right) \mb{t}^{k}_{i} + \frac{1}{m_i} \left\| \mb{z}^{k}_{i} - \mb{z}^* \right\|_2^2 \\\nonumber
&\leq \left( 1 - \frac{1}{M} \right) \mb{t}^{k}_{i} + \frac{2}{m} \left\| \mb{z}^{k}_{i} - \ol{\mb{x}}^{k} \right\|_2^2 + \frac{2}{m} \left\|\ol{\mb{x}}^{k} - \mb{z}^* \right\|_2^2,
\end{align}
where we use that~${\mb{v}_{i,j}^{k+1} = \mb{z}_i^{k}}$, with probability $\frac{1}{m_i}$, and~${\mb{v}_{i,j}^{k+1} = \mb{v}_i^{k}}$, with probability~${1-\frac{1}{m_i}}$, given~$\mc{F}^k$. We thus have 
\begin{align}\label{tk}
\mbb{E}[\mb{t}^{k+1} | \mc{F}^{k}] &\leq \left( 1 - \frac{1}{M} \right) \mb{t}^{k} + \frac{2}{m} \left\| \mb{z}^{k} - (\mb{1}_n \otimes \ol{\mb{x}}^{k}) \right\|_2^2 + \frac{2 n}{m} \left\|\ol{\mb{x}}^{k} - \mb{z}^* \right\|_2^2.
\end{align}

We consider the expansion of ${\|\mb{z}^{k} - (\mb{1}_n \otimes \ol{\mb{x}}^{k}) \|_2^2}$ next.
\begin{align} \label{ztox1}
\|\mb{z}^{k} - (\mb{1}_n \otimes \ol{\mb{x}}^{k}) \|_2^2
&= \|Y_k^{-1}[\mb{x}^{k} - Y^\infty (\mb{1}_n \otimes \ol{\mb{x}}^{k})] + [Y_k^{-1} Y^\infty - I_{n p}] (\mb{1}_n \otimes \ol{\mb{x}}^{k})\|_2^2\nonumber\\
&\leq y_-^2 \|\mb{x}^{k} - B^\infty \mb{x}^{k}\|_2^2 + (y_- T \lambda^k)^2 \|\mb{x}^{k}\|_2^2 + 2 y_- (y_- T \lambda^k) \|\mb{x}^{k} - B^\infty \mb{x}^{k}\|_2 \|\mb{x}^{k}\|_2 \nonumber\\
&\leq y_-^2 \|\mb{x}^{k} - B^\infty \mb{x}^{k}\|_2^2 + (y_- T \lambda^k)^2 \|\mb{x}^{k}\|_2^2 + y_-^2 T \lambda^k \left( \|\mb{x}^{k} - B^\infty \mb{x}^{k}\|_2^2 + \|\mb{x}^{k}\|_2^2 \right) \nonumber\\ 
&\leq (y_-^2 + y_-^2 T \lambda^k) \ol{\pi} \|\mb{x}^{k} - B^\infty \mb{x}^{k}\|_{\bpi}^2 + \left( y_- ^2 T^2 \lambda^{k} + y_-^2 T \lambda^k \right) \|\mb{x}^{k}\|_2^2,
\end{align}
where we used the Young's inequality and Lemma~\ref{y_lem}. Let~${d_1 := y_-^2 (1+T)}$ and ${d_2 := y_-^2 T(1 + T)}$; using~\eqref{ztox1} in~\eqref{tk} and taking full expectation we get the final expression for $\mbb{E}[\mb{t}^{k+1}]$,
\begin{align} \label{tk2}
\mbb{E}[\mb{t}^{k+1}] 
&\leq \frac{2}{m} d_1 \ol{\pi} \mbb{E} \|\mb{x}^{k} - B^\infty \mb{x}^{k}\|_{\bpi}^2 + \frac{2n}{m} \mbb{E} \|\ol{\mb{x}}^{k} - \mb{z}^*\|_2^2 + \left( 1 - \frac{1}{M} \right) \mb{t}^{k} + \frac{2}{m} d_2 \lambda^k \mbb{E} \|\mb{x}^{k}\|_2^2.
\end{align}
We define~$\psi := y y_-^2 (1+T)h$ as a directivity constant that quantifies the directed nature. In other words,~${\psi=1}$ for undirected graphs since~${y=y_-=1}, {T=0},$ and~${h=1}$ for doubly-stochastic weight matrices. Note that~$d_1 \leq \psi$~and~$d_2 \leq \psi T$. Hence, we can write
\begin{align} \label{tk3}
\mbb{E}[\mb{t}^{k+1}] 
&\leq \frac{2 \psi \ol{\pi}}{m} \mbb{E} \|\mb{x}^{k} - B^\infty \mb{x}^{k}\|_{\bpi}^2 + \frac{2n}{m} \mbb{E} \|\ol{\mb{x}}^{k} - \mb{z}^*\|_2^2 + \left( 1 - \frac{1}{M} \right) \mb{t}^{k} + \frac{2 \psi}{m} T\lambda^k \mbb{E} \|\mb{x}^{k}\|_2^2,
\end{align}
which is used to obtain the LTI system in~\eqref{sys_eq}.
\paragraph{Step 3 -- Optimality gap:} 
We now consider~$\mbb{E}\|\ol{\mb{x}}^{k+1} - \mb{z}^* \|_2^2$. It can be verified that
\[
{\mbb{E}[\|\ol{\mb{g}}^{k}\|^2_2 | \mc{F}^{k}] = \mbb{E}[\|\ol{\mb{g}}^{k} - \ol{\mb{h}}^{k}\|_2^2 | \mc{F}^{k}] + \|\ol{\mb{h}}^{k}\|_2^2},
\]
we thus have
\begin{align} \label{opt_gap1}
\mbb{E}[\|\ol{\mb{x}}^{k+1} - \mb{z}^* \|_2 ^2 | \mc{F}^{k}] &= \mbb{E}[\|\ol{\mb{x}}^{k} - \alpha \ol{\mb{w}}^k - \mb{z}^* \|_2 ^2 | \mc{F}^{k}] \nonumber\\
&\leq \|\ol{\mb{x}}^{k}- \mb{z}^*\|_2^2 + \alpha^2 \left( \mbb{E}[\|\ol{\mb{g}}^{k} - \ol{\mb{h}}^{k}\|_2^2 | \mc{F}^{k}] + \|\ol{\mb{h}}^{k} \|_2^2 \right) - 2 \alpha \langle \ol{\mb{x}}^{k}- \mb{z}^*, \ol{\mb{h}}^{k} \rangle \nonumber\\
&= \|\ol{\mb{x}}^{k}- \mb{z}^*\|_2^2 - 2 \alpha \langle \ol{\mb{x}}^{k}- \mb{z}^*, \ol{\mb{h}}^{k} \rangle + \alpha^2 \|\ol{\mb{h}}^{k} \|_2^2 + \alpha^2 \mbb{E}[\|\ol{\mb{g}}^{k} - \ol{\mb{h}}^{k}\|_2^2 | \mc{F}^{k}] \nonumber\\
&= \|\ol{\mb{x}}^{k}- \mb{z}^*\|_2^2 - 2 \alpha \langle \ol{\mb{x}}^{k}- \mb{z}^*, \ol{\mb{p}}^{k} \rangle + 2 \alpha \langle \ol{\mb{x}}^{k}- \mb{z}^*, \ol{\mb{p}}^{k} - \ol{\mb{h}}^{k} \rangle \nonumber\\
&~~~+ \alpha^2 \|\ol{\mb{p}}^{k} - \ol{\mb{h}}^{k}\|_2^2 + \alpha^2 \|\ol{\mb{p}}^{k} \|_2^2 - 2 \alpha^2 \langle \ol{\mb{p}}^{k}, \ol{\mb{p}}^{k} - \ol{\mb{h}}^{k} \rangle + \alpha^2 \mbb{E}[\|\ol{\mb{g}}^{k} - \ol{\mb{h}}^{k}\|_2^2 | \mc{F}^{k}] \nonumber\\
&= \|\ol{\mb{x}}^{k} - \alpha \ol{\mb{p}}^{k} - \mb{z}^*\|_2^2 + \alpha^2 \|\ol{\mb{p}}^{k} - \ol{\mb{h}}^{k}\|_2^2 + 2 \alpha \langle \ol{\mb{x}}^{k} - \alpha \ol{\mb{p}}^{k} - \mb{z}^*, \ol{\mb{p}}^{k} - \ol{\mb{h}}^{k} \rangle \nonumber\\
&~~~+ \alpha^2 \mbb{E}[\|\ol{\mb{g}}^{k} - \ol{\mb{h}}^{k}\|_2^2 | \mc{F}^{k}] \nonumber\\
&\leq (1-\alpha \mu)^2 \|\ol{\mb{x}}^{k} - \mb{z}^*\|_2^2 + \alpha^2 \|\ol{\mb{p}}^{k} - \ol{\mb{h}}^{k}\|_2^2 \nonumber\\
&~~~+ 2 \alpha (1-\alpha \mu)  \|\ol{\mb{x}}^{k} - \mb{z}^*\| \|\ol{\mb{p}}^{k} - \ol{\mb{h}}^{k}\| + \alpha^2 \mbb{E}[\|\ol{\mb{g}}^{k} - \ol{\mb{h}}^{k}\|_2^2 | \mc{F}^{k}] \nonumber\\
&\leq (1-\alpha \mu)^2 \|\ol{\mb{x}}^{k} - \mb{z}^*\|_2^2 + \alpha^2 \|\ol{\mb{p}}^{k} - \ol{\mb{h}}^{k}\|_2^2 \nonumber\\
&~~~+ \alpha (1-\alpha \mu) \left(\mu \|\ol{\mb{x}}^{k} - \mb{z}^*\|^2 +\frac{1}{\mu}\|\ol{\mb{p}}^{k} - \ol{\mb{h}}^{k}\|^2\right) + \alpha^2 \mbb{E}[\|\ol{\mb{g}}^{k} - \ol{\mb{h}}^{k}\|_2^2 | \mc{F}^{k}] \nonumber\\
&\leq (1 - \alpha \mu) \|\ol{\mb{x}}^{k} - \mb{z}^*\|_2^2 + \left( \frac{\alpha L^2}{n \mu}\right) \|\mb{z}^{k} - (\mb{1}_n \otimes \ol{\mb{x}}^{k})\|_2^2 + \alpha^2 \mbb{E}[\|\ol{\mb{g}}^{k} - \ol{\mb{h}}^{k}\|_2^2 | \mc{F}^{k}],
\end{align}
where we use~${\|\ol{\mb x}^k - \alpha \ol{\mb p}^k - \mb z^*\| \leq (1- \alpha \mu) \|\ol{\mb x}^k -  \mb z^*\|}$, if~${0< \alpha \leq \frac{1}{L}}$ and~$L$-smoothness of $f_i$'s, from Assumption 2. The expansion of~$\mbb{E}[\|\ol{\mb{g}}^{k} - \ol{\mb{h}}^{k}\|_2^2 | \mc{F}^{k}]$~follows:
\begin{align} \label{gh}
    \mbb{E}[\|\ol{\mb{g}}^{k} - \ol{\mb{h}}^{k}\|_2^2 | \mc{F}^{k}] 
    = \tfrac{1}{n^2} \mbb{E} [ \textstyle\sum_{i=1}^{n}  \|\mb{g}^{k}_{i}  - \nabla f_i(\mb{z}^{k}_{i}) \|_2^2 | \mc{F}^{k} ]= \frac{1}{n^2} \mbb{E} \left[ \|\mb{g}^{k}  - \nabla \mb f(\mb{z}^{k}) \|_2^2 | \mc{F}^{k} \right],
\end{align}
where we use~${\mbb{E}[\textstyle\sum_{i \neq j}  \langle \mb{g}^{k}_{i}  - \nabla f_i(\mb{z}^{k}_{i}), \mb{g}^{k}_j  - \nabla f_j(\mb{z}^{k}_j) \rangle | \mc{F}^{k}] = 0},$ since~$\mb g_k^i$'s are independent across the nodes. 

Next we define~${\ol{\nabla}_i^k:=\frac{1}{m_i}\sum_{j=1}^{m_i} \nabla f_{i,j}(\mb{v}^{k}_{i,j})}$~and obtain a bound on~$\mbb{E} \left[ \|\mb{g}^{k}  - \nabla \mb f(\mb{z}^{k}) \|_2^2 | \mc{F}^{k} \right]$ and start with local quantity.
\begin{align*}
     \mbb{E}\left[ \left \|\mb{g}^{k}_{i}  - \nabla f_i(\mb{z}^{k}_{i}) \right\|_2^2 \Big| \mc{F}^{k} \right] &= \mbb{E} \left[ \left \|\nabla f_{i,s_i^k}(\mb{z}^{k}_{i}) - \nabla f_{i,s_i^k}(\mb{v}^{k}_{i,j}) + \ol{\nabla}_i^k  - \nabla f_i(\mb{z}^{k}_{i}) \right\|_2^2 \Big| \mc{F}^{k} \right] \\
    &= \mbb{E} \Big[ \Big\| \underbrace{\nabla f_{i,s_i^k}(\mb{z}^{k}_{i}) - \nabla f_{i,s_i^k}(\mb{z}^*)}_{X_i^k} - (\underbrace{\nabla f_i(\mb{z}^{k}_{i}) - \nabla f_{i}(\mb{z}^*)}_{\mbb E [X_i^k]}) \\
    &~~~- \big(~ \underbrace{\nabla f_{i,s_i^k}(\mb{v}^{k}_{i,s^k}) - \nabla f_{i,s_i^k}(\mb{z}^*)}_{Z_i^k} - (~\underbrace{\ol{\nabla}_i^k - \nabla f_{i}(\mb{z}^*)}_{\mbb E [Z_i^k]}~)~\big) \Big\|_2^2  \Big| \mc{F}^{k} \Big]\\  &\leq 2 \mbb{E} \Big[ \Big\|\nabla f_{i,s_i^k}(\mb{z}^{k}_{i}) - \nabla f_{i,s_i^k}(\mb{z}^*) \Big \|_2^2 \Big| \mc{F}^{k} \Big] - 2 \|\nabla f_i(\mb{z}^{k}_{i}) - \nabla f_{i}(\mb{z}^*)\|_2^2  \\
    &~~~+ 2 \mbb{E} \Big[ \Big\| \nabla f_{i,s_i^k}(\mb{v}^{k}_{i,s_i^k}) - \nabla f_{i,s_i^k}(\mb{z}^*)  \Big \|_2^2 \Big| \mc{F}^{k} \Big] - 2 \| \ol{\nabla}_i^k - \nabla f_{i}(\mb{z}^*)  \|_2^2,
\end{align*}
where we use variance decomposition and the Young's Inequality. We drop the negative terms and further proceed as follows: 
\begin{align*}
     \mbb{E}\left[ \left \|\mb{g}^{k}_{i}  - \nabla f_i(\mb{z}^{k}_{i}) \right\|_2^2 \Big| \mc{F}^{k} \right] &\leq \frac{2}{m_i} \sum_{j=1}^{m_i} \Big( \big\|\nabla f_{i,j}(\mb{z}^{k}_{i}) - \nabla f_{i,j}(\mb{z}^*) \big \|_2^2 + \big\| \nabla f_{i,j}(\mb{v}^{k}_{i,j}) - \nabla f_{i,j}(\mb{z}^*)  \big \|_2^2 \Big) \\
    &\leq \frac{2 L^2}{m_i} \sum_{j=1}^{m_i}\left(2 \|\mb{z}^{k}_{i} - \ol{\mb{x}}^{k} \|_2^2 + 2 \|\ol{\mb{x}}^{k} - \mb{z}^* \|_2^2 \right) + \frac{2 L^2}{m_i} \sum_{j=1}^{m_i} \| \mb{v}^{k}_{i,j} - \mb{z}^* \|_2^2 \\
    &= \frac{4 L^2}{m_i} \sum_{j=1}^{m_i} \left( \|\mb{z}^{k}_{i} - \ol{\mb{x}}^{k} \|_2^2 + \|\ol{\mb{x}}^{k} - \mb{z}^* \|_2^2 \right) + 2 L^2 \mb{t}^{k}_{i},
\end{align*}
with the help of the $L$-smoothness of~$f_{i,j}$'s by Assumption 2, which leads to \begin{align}\label{gk_fzk}
    \mbb{E} \left[ \left \|\mb{g}^{k}  - \nabla \mb f(\mb{z}^{k}) \right\|_2^2 \Big| \mc{F}^{k} \right] &\leq 4 L^2 \|\mb{z}^{k} - (\mb{1}_n \otimes \ol{\mb{x}}^{k}) \|_2^2 + 4 n L^2 \|\ol{\mb{x}}^{k} - \mb{z}^* \|_2^2  + 2 L^2 \mb{t}^{k}.
\end{align}
Using the above in~\eqref{gh} and, we simplify~\eqref{opt_gap1} as
\begin{align*}
\mbb{E}[\|\overline{\mb{x}}^{k+1} - \mb{z}^* \|_2 ^2 | \mc{F}^{k}] &\leq (1 - \alpha \mu) \|\overline{\mb{x}}^{k} - \mb{z}^*\|_2^2 + \left( \frac{\alpha L^2}{n \mu}\right) \| \mb{z}^{k} - (\mb{1}_n \otimes \ol{\mb{x}}^{k})\|_2^2 \\
&~~~+ \alpha^2 \left( \frac{4 L^2}{n^2} \|\mb{z}^{k} - (\mb{1}_n \otimes \ol{\mb{x}}^{k}) \|_2^2 + \frac{4 L^2}{n} \|\ol{\mb{x}}^{k} - \mb{z}^* \|_2^2  + \frac{2 L^2}{n^2} \mb{t}^{k} \right) \\
&= \left(1 - \alpha \mu + \alpha^2 \frac{4 L^2}{n}\right) \|\overline{\mb{x}}^{k} - \mb{z}^*\|_2^2 + \alpha^2 \frac{2 L^2}{n} \mb{t}^{k}\\
&~~~+ \left( \frac{\alpha L^2}{n}\right) \left(\frac{1}{\mu} + \frac{4\alpha}{n} \right)  \|\mb{z}^{k} - (\mb{1}_n \otimes \ol{\mb{x}}^{k})\|_2^2 \\
&\leq \left(1 - \frac{\alpha \mu}{2} \right) \|\overline{\mb{x}}^{k} - \mb{z}^*\|_2^2 + \left( \frac{2 L^2 \alpha}{\mu n} \right)  \|\mb{z}^{k} - (\mb{1}_n \otimes \ol{\mb{x}}^{k})\|_2^2 + \alpha^2 \frac{2 L^2}{n^2} \mb{t}^{k}. 
\end{align*}
We note that the last step follows when we bound~${0< \alpha \leq \frac{n \mu}{8 L^2}}$~for the first term and~${0<\alpha \leq \frac{n}{4 \mu}}$~for the second term. We plug the expression for $\|\mb{z}^{k} - (\mb{1}_n \otimes \ol{\mb{x}}^{k}) \|_2^2$ from~\eqref{ztox1}~in the above equation and take full expectation for the final expression of optimality gap.
\begin{align*}
\mbb{E}[\|\overline{\mb{x}}^{k+1} - \mb{z}^* \|_2 ^2] &\leq \left( \frac{2 L^2 \alpha}{\mu n} \right)  d_1 \ol{\pi} \mbb{E}[\|\mb{x}^{k} - B^\infty \mb{x}^{k} \|^2 _{\bpi}] + \left(1 - \frac{\alpha \mu}{2} \right) \mbb{E}[\|\overline{\mb{x}}^{k} - \mb{z}^*\|_2^2] \\
&~~~ + \left( \frac{2 \alpha^2 L^2}{n^2} \right) \mbb{E}[\mb{t}^{k}] + \left( \frac{2 L^2 \alpha}{\mu n} \right) d_2 \lambda^k \mbb{E}[\|\mb{x}^{k} \|^2 _2]
\end{align*}
By using~${d_1 \leq \psi}$ and~${d_2 \leq \psi T}$, we further obtain
\begin{align*}
\mbb{E}[\|\overline{\mb{x}}^{k+1} - \mb{z}^* \|_2 ^2] &\leq \left( \frac{2  \alpha L^2 \psi \ol{\pi}}{n \mu} \right) \mbb{E}[\|\mb{x}^{k} - B^\infty \mb{x}^{k} \|^2 _{\bpi}] + \left(1 - \frac{\alpha \mu}{2} \right) \mbb{E}[\|\overline{\mb{x}}^{k} - \mb{z}^*\|_2^2] \\
&~~~+ \left(\frac{2 \alpha^2 L^2}{n^2}\right) \mbb{E}[\mb{t}^{k}] + \left( \frac{2 \alpha L^2 \psi}{n \mu} \right) T \lambda^k \mbb{E}[\|\mb{x}^{k} \|^2 _2],
\end{align*}
which is used to obtain the LTI system in~\eqref{sys_eq}.

Additionally, the following bound on the optimality gap is used in the final step of the derivation. 
\begin{align} \label{opt_gap2}
\mbb{E}[\|\ol{\mb{x}}^{k+1} - \mb{z}^* \|_2 ^2 | \mc{F}^{k}] &\leq (1-\alpha \mu)^2 \|\ol{\mb{x}}^{k} - \mb{z}^*\|_2^2 + \alpha^2 \|\ol{\mb{p}}^{k} - \ol{\mb{h}}^{k}\|_2^2 + \alpha^2 \mbb{E}[\|\ol{\mb{g}}^{k} - \ol{\mb{h}}^{k}\|_2^2 | \mc{F}^{k}] \nonumber\\
&~~~+ 2 \alpha (1-\alpha \mu)  \|\ol{\mb{x}}^{k} - \mb{z}^*\| \|\ol{\mb{p}}^{k} - \ol{\mb{h}}^{k}\| \nonumber\\
&\leq (1-\alpha \mu)^2 \|\overline{\mb{x}}^{k} - \mb{z}^*\|_2^2 + \alpha^2 \|\overline{\mb{p}}^{k} - \overline{\mb{h}}^{k}\|_2^2 + \alpha^2 \mbb{E}[\|\ol{\mb{g}}^{k} - \ol{\mb{h}}^{k}\|_2^2 | \mc{F}^{k}] \nonumber\\
&~~~+ (1-\alpha \mu) \left( \|\overline{\mb{x}}^{k} - \mb{z}^*\|^2 +\alpha^2\|\overline{\mb{p}}^{k} - \overline{\mb{h}}^{k}\|^2\right) \nonumber\\
&\leq 2 \|\overline{\mb{x}}^{k} - \mb{z}^*\|_2^2 + \frac{2 \alpha^2 L^2}{n} \|(\mb{1}_n \otimes \ol{\mb{x}}^{k}) - \mb{z}^{k}\|_2^2 + \alpha^2 \mbb{E}[\|\ol{\mb{g}}^{k} - \ol{\mb{h}}^{k}\|_2^2 | \mc{F}^{k}] \nonumber\\
&\leq 2 \|\overline{\mb{x}}^{k} - \mb{z}^*\|_2^2 + \frac{2 \alpha^2 L^2}{n} \|(\mb{1}_n \otimes \ol{\mb{x}}^{k}) - \mb{z}^{k}\|_2^2 \nonumber\\
&~~~+ \alpha^2 \left(\frac{4 L^2}{n^2} \|\mb{z}^{k} - (\mb{1}_n \otimes \ol{\mb{x}}^{k}) \|_2^2 + \frac{4 L^2}{n} \|\ol{\mb{x}}^{k} - \mb{z}^* \|_2^2  + \frac{2 L^2}{n^2} \mb{t}^{k} \right) \nonumber\\
&\leq \left(2 + \frac{4 \alpha^2 L^2}{n}\right) \|\overline{\mb{x}}^{k} - \mb{z}^*\|_2^2 + \frac{2 \alpha^2 L^2}{n^2} \mb{t}^{k} \nonumber\\\nonumber
&~~~+ \left( \frac{6 \alpha^2 L^2}{n} \right) \left(d_1 \ol{\pi} \|\mb{x}^{k} - B^\infty \mb{x}^{k}\|_{\bpi}^2 + d_2 \lambda^k \|\mb{x}^{k}\|_2^2 \right).
\end{align}
Multiplying both sides by~$n$~and using~${0<\alpha \leq \frac{\sqrt{n}}{2 L}}$~for the first term leads to
\begin{align}\nonumber\\
\mbb{E}[n \|\overline{\mb{x}}^{k+1} - \mb{z}^* \|_2 ^2 | \mc{F}^{k}] &\leq 3 n \|\overline{\mb{x}}^{k} - \mb{z}^*\|_2^2 + 6 \alpha^2 L^2 d_1 \ol{\pi} \|\mb{x}^{k} - B^\infty \mb{x}^{k}\|_{\bpi}^2 \nonumber\\
&~~~+ 6 \alpha^2 L^2 d_2 \lambda^k \|\mb{x}^{k}\|_2^2 + \frac{2 \alpha^2 L^2}{n} \mb{t}^{k}.
\end{align}

\paragraph{Step 4 -- Gradient tracking error} Finally, we find a bound on the gradient tracking error to complete the LTI system in~\eqref{sys_eq}. We have
\begin{align} \label{gt1}
   \mbb{E}\left[\|\mb{w}^{k+1} - B^\infty \mb{w}^{k+1} \|_{\bpi} ^2 | \mc{F}^k \right]
   &= \mbb{E}\left[\|B \mb{w}^{k} - B^\infty \mb{w}^{k} + (I_{n p} - B^\infty) (\mb{g}^{k+1} - \mb{g}^{k})\|_{\bpi}^2 | \mc{F}^k \right] \nonumber\\
   &\leq \left(1+\frac{1-\lambda^2}{2\lambda^2} \right) \lambda^2 \mbb{E}\left[\|\mb{w}^{k} - B^\infty \mb{w}^{k} \|_{\bpi}^2 | \mc{F}^k \right] \nonumber \\
   &~~~+ \left(1+\frac{2\lambda^2}{1-\lambda^2} \right) \mbb{E}\left[\|\mb{g}^{k+1} - \mb{g}^{k}\|_{\bpi}^2 | \mc{F}^k \right] \nonumber\\
   &\leq \frac{1 + \lambda^2}{2}  \mbb{E}\left[\|\mb{w}^{k} - B^\infty \mb{w}^{k} \|_{\bpi}^2 | \mc{F}^k \right] + \frac{2}{1 - \lambda^2}  \mbb{E}\left[\|\mb{g}^{k+1} - \mb{g}^{k}\|_{\bpi}^2 | \mc{F}^k \right]. 
\end{align}
where we use the Young's inequality in the second step. Next we expand the second term above.
\begin{align}\label{gg_eq}
   \mbb{E}\left[\|\mb{g}^{k+1} - \mb{g}^{k}\|_{\bpi}^2 | \mc{F}^k \right] &= \mbb{E}\left[\|\mb{g}^{k+1} - \mb{g}^{k} - \nabla \mb{f}(\mb{z}^{k+1}) + \nabla \mb{f}(\mb{z}^{k}) + \nabla \mb{f}(\mb{z}^{k+1}) - \nabla \mb{f}(\mb{z}^{k})\|_{\bpi}^2 | \mc{F}^k \right] \nonumber \\
   &\leq 2 \mbb E \left[\|\nabla \mb{f}(\mb{z}^{k+1}) - \nabla \mb{f}(\mb{z}^{k})\|_{\bpi}^2 | \mc{F}^k \right] \nonumber\\
   &~~~+ 2 \mbb E \left[\| \mb{g}^{k+1} - \mb{g}^{k} - \nabla \mb{f}(\mb{z}^{k+1}) + \nabla \mb{f}(\mb{z}^{k})\|_{\bpi}^2 | \mc{F}^k \right] \nonumber \\
   &\leq 2 L^2 \ul{\pi}^{-1} \mbb E \left[\|\mb{z}^{k+1} - \mb{z}^{k}\|_{2}^2 | \mc{F}^k \right] + 4 \ul{\pi}^{-1} \mbb{E}[\| \mb{g}^{k} - \nabla \mb{f}(\mb{z}^{k}) \|_{2}^2| \mc{F}^{k}] \nonumber\\
   &~~~+ 4 \ul{\pi}^{-1} \mbb{E}\left[\mbb{E}[\| \mb{g}^{k+1} - \nabla \mb{f}(\mb{z}^{k+1})\|_{2}^2| \mc{F}^{k+1}] \big| \mc{F}^{k} \right]
\end{align}
For~$0<\alpha \leq \frac{1}{4L\sqrt{6}}$, we first simplify~$\mbb{E}\left[\mbb{E}[\| \mb{g}^{k+1} - \nabla f(\mb{z}^{k+1})\|_{2}^2| \mc{F}^{k+1}] \big| \mc{F}^{k} \right]$ with the help of~\eqref{gk_fzk}:
\begin{align} \label{gfk1}
    &\mbb{E}\left[\mbb{E}[\| \mb{g}^{k+1} - \nabla f(\mb{z}^{k+1})\|_{2}^2| \mc{F}^{k+1}] \big| \mc{F}^{k} \right] \leq 4 L^2 \|\mb{z}^{k+1} - (\mb{1}_n \otimes \ol{\mb{x}}^{k+1}) \|_2^2 \nonumber\\
    &+ 4 n L^2 \|\ol{\mb{x}}^{k+1} - \mb{z}^* \|_2^2  + 2 L^2 \mb{t}^{k+1} \nonumber\\
    \leq &~4 L^2 \bigg ( d_1 \ol{\pi} \|\mb{x}^{k+1} - B^\infty \mb{x}^{k+1}\|_{\bpi}^2 + d_2 \lambda^k \|\mb{x}^{k+1}\|_2^2 \bigg) + 4 n L^2 \|\ol{\mb{x}}^{k+1} - \mb{z}^* \|_2^2  + 2 L^2 \mb{t}^{k+1} \nonumber\\
    \leq &~4 L^2 d_1 \ol{\pi} \left( 2 \|\mb{x}^{k} - B^\infty \mb{x}^{k}\|^2 _{\bpi} + 2 \alpha^2 \| \mb{w}^{k} - B^\infty \mb{w}^{k} \|^2 _{\bpi} \right) + 4 L^2 \bigg ( 6 L^2 \alpha^2 d_1 \ol{\pi} \|\mb{x}^{k} - B^\infty \mb{x}^{k} \|^2 _{\bpi} \nonumber\\
    &+ 3 n \|\overline{\mb{x}}^{k} - \mb{z}^*\|_2^2 + 6 L^2 \alpha^2 d_2 \lambda^k \|\mb{x}^{k} \|^2 _2  + \frac{2 \alpha^2 L^2}{n} \mb{t}^{k} \bigg ) + 4 L^2 d_2 \lambda^k \left( \|B \mb{x}^{k} - \alpha \mb{w}^k \|^2 _2 \right) \nonumber\\
    &+ 2 L^2 \bigg( \left( 1 - \frac{1}{M} \right) \mb{t}^{k} + \frac{2}{m} d_1 \ol{\pi} \|\mb{x}^{k} - B^\infty \mb{x}^{k}\|_{\bpi}^2 + \frac{2 n}{m} \left\|\ol{\mb{x}}^{k} - \mb{z}^* \right\|_2^2 + \frac{2}{m} d_2 \lambda^k \|\mb{x}^{k}\|_2^2 \bigg) \nonumber\\
    \leq & ~12.25 L^2 d_1 \ol{\pi} \|\mb{x}^{k} - B^\infty \mb{x}^{k}\|_{\bpi}^2 +  16 L^2 n \|\overline{\mb{x}}^{k} - \mb{z}^*\|_2^2 + 8 L^2 \alpha^2 d_1 \ol{\pi} \| \mb{w}^{k} - B^\infty \mb{w}^{k} \|^2 _{\bpi} + 2.25 L^2 \mb{t}^k \nonumber\\
    &+ 12.25 L^2 d_2 \lambda^k \|\mb{x}^{k} \|^2 _2 + 8 L^2 \alpha^2 d_2 \lambda^k \|\mb{w}^k \|^2 _2.
\end{align}
Similarly, we simplify~$\mbb{E}\left[\| \mb{g}^{k} - \nabla f(\mb{z}^{k})\|_{2}^2 \big| \mc{F}^{k} \right]$:
\begin{align} \label{gfk}
\mbb{E}\left[\| \mb{g}^{k} - \nabla f(\mb{z}^{k})\|_{2}^2 \big| \mc{F}^{k} \right] &\leq 4 L^2 \|\mb{z}^{k} - (\mb{1}_n \otimes \ol{\mb{x}}^{k}) \|_2^2 + 4 n L^2 \|\ol{\mb{x}}^{k} - \mb{z}^* \|_2^2  + 2 L^2 \mb{t}^{k} \nonumber\\
&\leq 4 L^2 d_1 \ol{\pi} \|\mb{x}^{k} - B^\infty \mb{x}^{k}\|_{\bpi}^2 + 4 L^2 d_2 \lambda^k \|\mb{x}^{k}\|_2^2 \nonumber\\
&~~~+ 4 n L^2 \|\ol{\mb{x}}^{k} - \mb{z}^* \|_2^2  + 2 L^2 \mb{t}^{k}.
\end{align}

The final term to be expanded in~\eqref{gg_eq}~can be written as
\begin{align*}
    \|\mb{z}^{k+1} - \mb{z}^{k}\|_{2}^2
    &= \|Y_{k+1}^{-1}((B \mb{x}^{k} - \alpha \mb{w}^{k}) - \mb{x}^{k}) + (Y_{k+1}^{-1} - Y_{k}^{-1}) \mb{x}^{k} \|_{2}^2 \nonumber\\
    &= \|Y_{k+1}^{-1}(B - I_{n})\mb{x}^{k} - \alpha Y_{k+1}^{-1} \mb{w}^{k} + (Y_{k+1}^{-1} - Y_{k}^{-1}) \mb{x}^{k} \|_{2}^2 \nonumber\\
    &\leq \|Y_{k+1}^{-1}(B - I_{n})\mb{x}^{k}\|_{2}^2 +  \|\alpha Y_{k+1}^{-1} \mb{w}^{k}\|_{2}^2 + \mn{Y_{k+1}^{-1} - Y_{k}^{-1}}_{2}^2 \|\mb{x}^{k} \|_{2}^2 \nonumber\\ 
    &~~~+ 2 \|Y_{k+1}^{-1}(B - I_{n})\mb{x}^{k}\|_{2}\| \alpha Y_{k+1}^{-1} \mb{w}^{k}\|_{2} + 2 \|\alpha Y_{k+1}^{-1} \mb{w}^{k}\|_{2}\mn{Y_{k+1}^{-1} - Y_{k}^{-1}}_{2} \|\mb{x}^{k} \|_{2} \nonumber\\
    &~~~+ 2\|Y_{k+1}^{-1}(B - I_{n})\mb{x}^{k}\|_{2}\mn{Y_{k+1}^{-1} - Y_{k}^{-1}}_{2} \|\mb{x}^{k} \|_{2} \nonumber \\
    &\leq 12 y_- ^2 \ol{\pi} \| \mb{x}^{k} - B^\infty \mb{x}^{k} \|^2_{\bpi} + 3 \alpha^2 y_-^2 \| \mb{w}^{k} \|_{2}^2 + 12 y_-^4 T^2 \lambda^{2k} \|\mb{x}^{k}\|_2^2, 
    \end{align*}
    which leads to
    \begin{align} \label{zzk1}
    \mbb{E}\left[\|\mb{z}^{k+1} - \mb{z}^{k}\|_{2}^2 \big| \mc{F}^{k} \right] &\leq 12 y_- ^2 \ol{\pi} \| \mb{x}^{k} - B^\infty \mb{x}^{k} \|^2_{\bpi} + 3 \alpha^2 y_-^2 \mbb{E}\left[\| \mb{w}^{k} \|_{2}^2\big| \mc{F}^{k} \right] + 12 y_-^4 T^2 \lambda^{2k} \|\mb{x}^{k}\|_2^2.
\end{align}
This leaves us with the following bound left to be established:
\begin{align}\label{wk}
    \mbb{E}\left[\|\mb{w}^{k}\|_{2}^2\big| \mc{F}^{k} \right] &=\mbb{E}\left[\|\mb{w}^{k} - Y^\infty (\mb{1}_n \otimes \overline{\mb{g}}^{k}) + Y^\infty (\mb{1}_n \otimes \overline{\mb{g}}^{k}) -Y^\infty (\mb{1}_n \otimes \overline{\mb{p}}^{k}) + Y^\infty (\mb{1}_n \otimes \overline{\mb{p}}^{k}) \|_{2}^2 | \mc{F}^{k} \right] \nonumber\\ 
    &\leq 3 \mbb{E}\left[\|\mb{w}^{k} - B^\infty \mb{w}^{k} \|_{2}^2 \big| \mc{F}^{k} \right] + 3 y_-^2 y^2 \|\mb{1}_n  \otimes \overline{\mb{p}}^{k}\|_{2}^2 + 3 y_-^2 y^2 n \mbb{E}\left[ \|\overline{\mb{g}}^{k} - \overline{\mb{p}}^{k} \|_{2}^2\big| \mc{F}^{k} \right] \nonumber\\
    &\leq 3 \ol{\pi} \mbb{E}\left[\|\mb{w}^{k} - B^\infty \mb{w}^{k} \|_{\bpi}^2 \big| \mc{F}^{k} \right] + 3 y_-^2 y^2 L^2  \|\ol{\mb{x}}^{k} - \mb{z}^*\|_{2}^2 \nonumber\\
    &~~~+6 y_-^2 y^2 n \mbb{E}\left[\|\ol{\mb{g}}^{k} - \ol{\mb{h}}^{k} \|_2^2\big| \mc{F}^{k} \right] + 6 y_-^2 y^2 L^2 \|\mb{z}^{k} - (\mb{1}_n \otimes \ol{\mb{x}}^{k}) \|_2^2 \nonumber\\
    &\leq 3 \ol{\pi} \mbb{E}\left[\|\mb{w}^{k} - B^\infty \mb{w}^{k} \|_{\bpi}^2 \big| \mc{F}^{k} \right] + 3 y_-^2 y^2 L^2  \|\ol{\mb{x}}^{k} - \mb{z}^*\|_{2}^2 \nonumber\\
    &~~~+ 6 y_-^2 y^2 n \left( \frac{4 L^2}{n^2} \|\mb{z}^{k} - (\mb{1}_n \otimes \ol{\mb{x}}^{k}) \|_2^2 + \frac{4 L^2}{n} \|\ol{\mb{x}}^{k} - \mb{z}^* \|_2^2  + \frac{2 L^2}{n^2} \mb{t}^{k} \right) \nonumber\\
    &~~~+ 6 y_-^2 y^2 L^2 \|\mb{z}^{k} - (\mb{1}_n \otimes \ol{\mb{x}}^{k}) \|_2^2 \nonumber\\
    &\leq 3 \ol{\pi} \mbb{E}\left[\|\mb{w}^{k} - B^\infty \mb{w}^{k} \|_{\bpi}^2 \big| \mc{F}^{k} \right] + 27 y_-^2 y^2 L^2  \|\ol{\mb{x}}^{k} - \mb{z}^*\|_{2}^2 + \frac{12 y_-^2 y^2 L^2}{n} \mb{t}_k \nonumber\\
    &~~~+ \left(6 y_-^2 y^2 L^2 + \frac{24 y_-^2 y^2 L^2}{n} \right) \left( d_1 \ol{\pi} \|\mb{x}^{k} - B^\infty \mb{x}^{k}\|_{\bpi}^2 + d_2 \lambda^k \|\mb{x}^{k}\|_2^2 \right) \nonumber\\
    &\leq 3 \ol{\pi} \mbb{E}\left[\|\mb{w}^{k} - B^\infty \mb{w}^{k} \|_{\bpi}^2 \big| \mc{F}^{k} \right] + 27 y_-^2 y^2 L^2  \|\ol{\mb{x}}^{k} - \mb{z}^*\|_{2}^2 + \frac{12 y_-^2 y^2 L^2}{n} \mb{t}_k \nonumber\\
    &~~~+ 30 y_-^2 y^2 L^2 d_1 \ol{\pi} \|\mb{x}^{k} - B^\infty \mb{x}^{k}\|_{\bpi}^2 + 30 y_-^2 y^2 L^2 d_2 \lambda^k \|\mb{x}^{k}\|_2^2 
\end{align}
We use the above expression of~$\|\mb{w}^{k}\|_{2}^2$~in~\eqref{gfk1}~and~\eqref{zzk1}~followed by plugging the evaluated expressions of~\eqref{gfk1},~\eqref{gfk}~and~\eqref{zzk1}~in~\eqref{gg_eq}. The final expression of~\eqref{gt1} is
\begin{align*}
    &\mbb{E}\left[\|\mb{w}^{k+1} -B^\infty \mb{w}^{k+1} \|^2 _{\bpi} |\mc{F}^{k} \right] \\
    \leq &\left( \frac{48 y_-^2 + 130 d_1}{1 - \lambda^2} + \frac{120 \alpha^2 L^2 y_-^2 y^2 d_1 \left( 3 y_-^2 + 16 d_2 \right)}{1 - \lambda^2} \right) L^2 h \| \mb{x}^{k} - B^\infty \mb{x}^{k} \|^2_{\bpi} \\
    +& \left(\frac{160 }{1 - \lambda^2} + \frac{108 \alpha^2 L^2 y_-^2 y^2 \left( 3 y_-^2 + 16 d_2 \right)}{n(1 - \lambda^2)} \right) L^2 \ul{\pi}^{-1} n \|\overline{\mb{x}}^{k} - \mb{z}^*\|_2^2 \\
    +& \left(\frac{34}{1 - \lambda^2} + \frac{48 \alpha^2 L^2 y_-^2 y^2 \left( 3 y_-^2 + 16 d_2 \right)}{n(1 - \lambda^2)} \right) L^2 \ul{\pi}^{-1} \mb{t}^k  \\
    +& \left(\frac{48 y_-^4 T^2 \lambda^{2k} + 130 d_2 \lambda^k}{1 - \lambda^2} + \frac{120 \alpha^2 L^2 y_-^2 y^2 d_2 \lambda^k \left( 3 y_-^2 + 16 d_2 \lambda^k \right)}{1 - \lambda^2} \right)  L^2 \ul{\pi}^{-1} \|\mb{x}^{k}\|_2^2
    \\
    +& \left( \frac{1 + \lambda^2}{2} + \frac{4 \alpha^2 L^2 {h}_c (16 d_1 + 9 y_-^2 + 48 d_2)}{1 - \lambda^2}  \right) \mbb{E}\left[\|\mb{w}^{k} - B^\infty \mb{w}^{k} \|_{\bpi}^2|\mc{F}^{k} \right]
\end{align*}

Let~${0<\alpha^2 \leq \frac{1}{12 L^2 y_-^2 y^2 (3 y_-^2 + 16 d_2)}}$~and~${0<\alpha^2 \leq \left(\frac{1-\lambda^2}{2}\right)^2 \left(\frac{1}{4 L^2 h (16 d_1 + 9 y_-^2 + 48 d_2)}\right)}$ for the first four terms and for the last term, respectively. Furthermore, we note that \begin{align*}
y_- h &\leq d_1 h \leq \psi\qquad\mbox{and}\qquad
48 y_-^4 T^2 \ul{\pi}^{-1} + 140 d_2 \ul{\pi}^{-1} \leq 188 \psi^2 T.
\end{align*}
Taking full expectation, we get the final expression for the gradient tracking error as
\begin{align*}
    \mbb{E}\left[\|\mb{w}^{k+1} -B^\infty \mb{w}^{k+1} \|^2 _{\bpi} \right] &\leq \left( \frac{188 L^2 \psi }{1 - \lambda^2} \right) \mbb{E} [\| \mb{x}^{k} - B^\infty \mb{x}^{k} \|^2_{\bpi}] + \left(\frac{169 L^2 \ul{\pi}^{-1} }{1 - \lambda^2} \right)  n \mbb{E} [\|\overline{\mb{x}}^{k} - \mb{z}^*\|_2^2] \\
    &~~~+ \left(\frac{38 L^2 \ul{\pi}^{-1}}{1 - \lambda^2} \right)  \mbb{E} [\mb{t}^k] + \left( \frac{3 + \lambda^2}{4} \right) \mbb{E} [\|\mb{w}^{k} - B^\infty \mb{w}^{k} \|_{\bpi}^2] \\
    &~~~+ \left(\frac{188 L^2 \psi^2}{1 - \lambda^2} \right) T\lambda^k \mbb{E} [\|\mb{x}^{k}\|_2^2].
\end{align*}
Next we chose a step-size that is the smallest among all above~$\alpha$~bounds and thus have~${0<\alpha \leq \frac{1-\lambda^2}{28 L \kappa \psi}}$ to complete the LTI system and the lemma follows.
\end{proof}

\section{Proof of Lemma 3: Stability of~$G_\alpha$}
\begin{proof}We recall from~\cite{hornjohnson} that for a non-negative matrix~$G_\alpha$, if there exist a positive vector~$\bds \delta$ and a positive constant~$\gamma$ such that~$G_\alpha\bds\delta\leq\gamma\bds\delta$, then~$\rho(G_\alpha) \leq \mn{G_\alpha }_\infty^{\bds\delta} \leq \gamma.$ Setting~${\gamma = \left( 1 - \frac{\alpha \mu}{4} \right)}$, we thus solve for the range of~$\alpha$ and a positive vector~${\bds{\delta} = [\delta_1, \delta_2, \delta_3, \delta_4]^\top}$ such that the inequalities~$G_\alpha \bds{\delta} \leq (1 - \frac{\alpha \mu}{4} ) \bds{\delta}$ hold elementwise. With~$G_\alpha$ from Lemma~\ref{main_lem}, we obtain
\begin{align} \label{e1}
    &\frac{\alpha \mu}{4} + \frac{2 \alpha^2 L^2}{1-\lambda^2} \frac{\delta_4}{\delta_1} \leq \frac{1-\lambda^2}{2},\\ \label{e2}
    &\frac{2 \alpha L^2}{n} \delta_3 \leq \frac{\mu}{4} \delta_2 - \frac{2 L^2 \psi \ol{\pi}}{\mu} \delta_1,\\ \label{e3}
    &\frac{\alpha \mu}{4} \leq \frac{1}{M} - \frac{2 \psi \ol{\pi}}{m} \frac{\delta_1}{\delta_3} - \frac{2}{m} \frac{\delta_2}{\delta_3},\\ \label{e4}
    &\frac{\alpha \mu}{4} \leq \frac{1-\lambda^2}{4} - \frac{188 \psi}{1-\lambda^2} \frac{\delta_1}{\delta_4} - \frac{169 \ul{\pi}^{-1}}{1-\lambda^2} \frac{\delta_2}{\delta_4} - \frac{38 \ul{\pi}^{-1}}{1-\lambda^2} \frac{\delta_3}{\delta_4}.
\end{align}
We note that~\eqref{e2},~\eqref{e3}, and~\eqref{e4} are true for some feasible range of~$\alpha$ when their right hand sides are positive. We fix the elements of~$\bds{\delta}$, independent of~$\alpha$, such that~${\delta_1 = 1}, {\delta_2 = 8.5 \kappa^2 \psi \ol{\pi}}, {\delta_3 = \frac{20 M \kappa^2 \psi \ol{\pi}}{m}}$, and~${\delta_4 = \frac{19076M\kappa^2 \psi h}{m(1-\lambda^2)^2}}$. It can be verified for these~$\delta_1, \delta_2, \delta_3$, and~$\delta_4$, the right hand side of~\eqref{e2},~\eqref{e3} and~\eqref{e4} are positive. We next solve to find the range of~$\alpha$. From~\eqref{e2}, we have
\begin{align*}
    \alpha &\leq \frac{n}{2 L^2 \delta_3} \left(\frac{\mu}{4} \delta_2 - \frac{2 L^2 \psi \ol{\pi}}{\mu} \delta_1 \right) = \frac{m}{M}\frac{n}{320 L \kappa}.
\end{align*}
For~\eqref{e3}, the bound on~$\alpha$ follows:
\begin{align*}
    \alpha &\leq \frac{4}{\mu} \left(\frac{1}{M} - \frac{2 \psi \ol{\pi}}{m} \frac{\delta_1}{\delta_3} - \frac{2}{m} \frac{\delta_2}{\delta_3} \right) = \frac{4}{M \mu} \left(1 - \frac{1}{10 \kappa^2} - \frac{8.5}{10} \right).
\end{align*}
It is straightforward to verify that~${\alpha \leq \frac{1}{5M\mu}}$~satisfies the above relation. Solving for~\eqref{e4} leads to
\begin{align*}
    \alpha &\leq \frac{4}{\mu} \left(\frac{1-\lambda^2}{4} - \frac{188 \psi}{1-\lambda^2} \frac{\delta_1}{\delta_4} - \frac{169 \ul{\pi}^{-1}}{1-\lambda^2} \frac{\delta_2}{\delta_4} - \frac{38 \ul{\pi}^{-1}}{1-\lambda^2} \frac{\delta_3}{\delta_4} \right) \\
    &= \frac{4(1-\lambda^2)}{\mu} \left(\frac{1}{4} - \frac{188 m}{19076 M \kappa^2 h} - \frac{1436.5 m}{19076 M } - \frac{760}{19076} \right),
\end{align*}
which is true for~$\alpha \leq \frac{1-\lambda^2}{2 \mu}$.
Next, we note that~\eqref{e1} has a solution if we choose~$\alpha \leq \frac{1-\lambda^2}{\mu}$ for the first term and~$\alpha \leq \frac{m}{M} \frac{(1-\lambda^2)^2}{400 \kappa L \psi}$ for the second term. Finally, we use the minimum across all the bounds on~$\alpha$ and the lemma follows by putting~$\alpha$ in~$\gamma=1-\frac{\alpha\mu}{4}$ and noting that~$(1-\lambda) < (1-\lambda^2)$.
\end{proof}
\end{document}